\documentclass{article}

\usepackage{microtype}
\usepackage{graphicx}
\usepackage{subcaption}
\usepackage{booktabs} % for professional tables
\usepackage{tabularx}
\usepackage{booktabs}
\usepackage{array}
\usepackage{hyperref}       % hyperlinks
\usepackage{url}            % simple URL typesetting
\usepackage{booktabs}       % professional-quality tables
\usepackage{amsfonts}       % blackboard math symbols
\usepackage{nicefrac}       % compact symbols for 1/2, etc.
\usepackage{microtype}      % microtypography
\usepackage{xcolor}         % colors
\usepackage{wrapfig}
\usepackage{microtype}
\usepackage{graphicx}
\usepackage{booktabs} 
\usepackage{tabularx}
\usepackage{listings}
\usepackage{xcolor}

\definecolor{codegreen}{rgb}{0,0.6,0}
\definecolor{codegray}{rgb}{0.5,0.5,0.5}
\definecolor{codepurple}{rgb}{0.58,0,0.82}
\definecolor{backcolour}{rgb}{0.95,0.95,0.92}

\lstdefinestyle{mystyle}{
    backgroundcolor=\color{backcolour},   
    commentstyle=\color{codegreen},
    keywordstyle=\color{magenta},
    numberstyle=\tiny\color{codegray},
    stringstyle=\color{codepurple},
    basicstyle=\ttfamily\footnotesize,
    breakatwhitespace=false,         
    breaklines=true,                 
    captionpos=b,                    
    keepspaces=true,                 
    numbers=left,                    
    numbersep=5pt,                  
    showspaces=false,                
    showstringspaces=false,
    showtabs=false,                  
    tabsize=4
}

\usepackage{hyperref}

\usepackage[preprint]{icml2026}

\usepackage{amsmath}
\usepackage{amssymb}
\usepackage{mathtools}
\usepackage{amsthm}
\usepackage{autonum}
\usepackage{enumitem}
\usepackage{bm}

\newcommand{\divg}{\mathbb{D}}
 
\newcommand{\repp}{p} 
\newcommand{\rep}{h}
\newcommand{\bce}{\mathrm{BCE}} 

\usepackage{xcolor}

\usepackage[capitalize,noabbrev]{cleveref}

\theoremstyle{plain}
\newtheorem{theorem}{Theorem}[section]
\newtheorem{proposition}[theorem]{Proposition}

\theoremstyle{definition}
\newtheorem{definition}[theorem]{Definition}

\theoremstyle{remark}

\usepackage[textsize=tiny]{todonotes}

\icmltitlerunning{Latent Adversarial Regularization for Offline Preference Optimization}

\begin{document}

\twocolumn[
  \icmltitle{Latent Adversarial Regularization for Offline Preference Optimization}

  % Aligning LLMs via Latent-Space Adversarial Anchoring Preference Optimization

  % Aligning LLMs via Latent-Space Adversarial Regularization

  % Latent Adversarial Regularization for Preference Optimization
    
  % Latent Adversarial Regularization for Structure-Preserving Preference Optimization



  \icmlsetsymbol{equal}{*}

  \begin{icmlauthorlist}
    \icmlauthor{Enyi Jiang}{stanford,uiuc}
    \icmlauthor{Yibo Jacky Zhang}{stanford}
    \icmlauthor{Yinglun Xu}{uiuc}
    \icmlauthor{Andreas Haupt}{stanford}
    \icmlauthor{Nancy Amato}{uiuc}
    \icmlauthor{Sanmi Koyejo}{stanford}
  \end{icmlauthorlist}

  \icmlaffiliation{stanford}{Department of Computer Science, Stanford University}
  \icmlaffiliation{uiuc}{Siebel School of Computing and Data Science, University of Illinois at Urbana-Champaign}

  \icmlcorrespondingauthor{Enyi Jiang}{enyij2@illinois.edu, enyij@stanford.edu}

  \icmlkeywords{Machine Learning, ICML}

  \vskip 0.3in
]



\printAffiliationsAndNotice{}  % no special notice (required even if empty)

\begin{abstract}
  Learning from human feedback typically relies on preference optimization that constrains policy updates through token-level regularization. However, preference optimization for language models is particularly challenging because token-space similarity does not imply semantic or behavioral similarity. To address this challenge, we leverage latent-space regularization for language model preference optimization. We introduce GANPO, which achieves latent-space regularization by penalizing divergence between the internal representations of a policy model and a reference model. Given that latent representations are not associated with explicit probability densities, we adopt an adversarial approach inspired by GANs to minimize latent-space divergence. We integrate GANPO as a regularizer into existing offline preference optimization objectives. Experiments across multiple model architectures and tasks show consistent improvements from latent-space regularization. Further, by comparing GANPO-induced inferential biases with those from token-level regularization, we find that GANPO provides more robust structural feedback under distributional shift and noise while maintaining comparable downstream performance with minor computational overhead. Our code is available at~\url{https://github.com/enyijiang/GANPO}.

\end{abstract}

\section{Introduction}
% \vspace{-0.1cm}
\begin{figure*}
    \centering
    \includegraphics[width=0.95\linewidth]{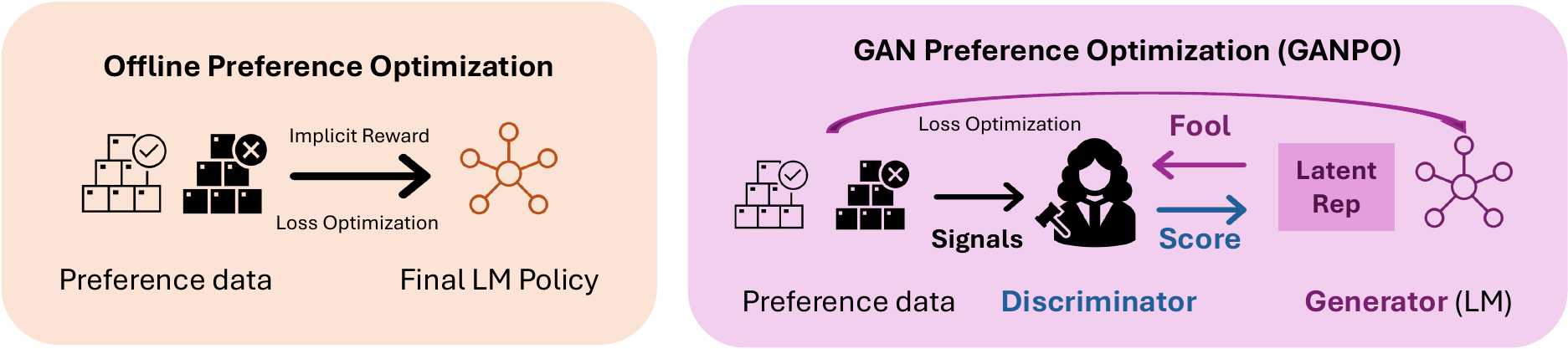}
    \caption{\textbf{Comparison between DPO and GANPO.}  
    Offline preference optimization methods (e.g., DPO) optimize an implicit reward defined by preference data. GANPO augments this objective with a latent-space discriminator, whose adversarial interaction induces a regularization between the latent representation distributions of the policy model and the reference model. }
    \label{fig:GANPO}
\end{figure*}

Learning from human feedback is essential for aligning large language models (LLMs) with human preferences~\cite{leike2018scalable,Ouyang2022TrainingLM}. Most modern methods are done via pairwise preference optimization (PO): given pairwise preference data, we update a policy model $\pi_\theta$ while constraining it to remain close to a reference model $\pi_{\text{ref}}$. This regularization constraint is crucial for improving generalization and reducing reward hacking, and it is most commonly implemented via a KL regularizer~\cite{Ouyang2022TrainingLM, rafailov2023dpo, xiong2024iterative}. Recent work has also explored alternative divergence measures, e.g., $\chi^2$-divergence and other $f$-divergences~\cite{wang2023beyond, huang2025correcting}.

\begin{figure}[!h]
    \centering
    \includegraphics[width=\linewidth]{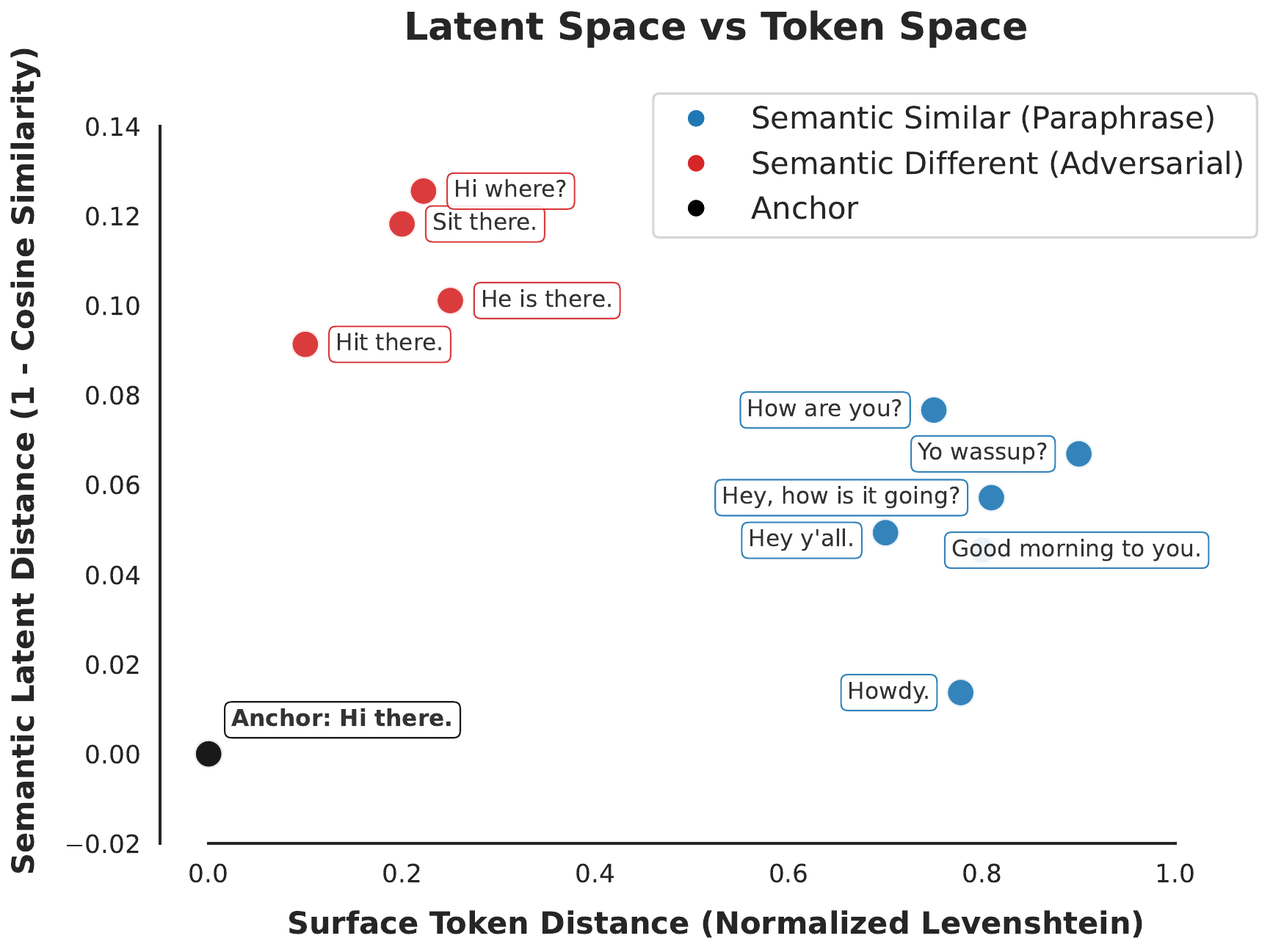}
    \caption{\textbf{Latent space vs token space.} Anchor (``Hi there.'') is the reference point for distance measurements. Semantically similar paraphrases exhibit large token-level variation yet remain close in latent space, while semantically different phrases show smaller token changes but larger latent space differences.}
    \label{fig:insight-latent-space}
    \vspace{-0.25in}
\end{figure}

However, a limitation of these regularizers is that they operate purely in token space. Two sentences, for example, may be distant under token-level divergences while remaining semantically very similar (``Hi there" vs ``Good morning to you") and vice versa (``Hi there" vs ``Hit there"), as illustrated in Figure~\ref{fig:insight-latent-space}. As a result, token-space divergences can be a coarse proxy for the actual behavioral similarity. Motivated by the limitation of token space optimization, recent work begun to explore the use of latent spaces instead of token spaces. For example, studies on continuous latent reasoning suggest that optimization in latent representations can lead to improved reasoning capabilities~\cite{hao2025training, zhu2025reasoning} compared to methods that purely operate in token space.

This paper investigates regularization performed in the latent space: we show that latent regularization can provide structural alignment feedback, which is absent from token-level constraints. Concretely, given a preference dataset, we consider the distributions of internal representations produced by the policy model and the reference model, and penalize their divergence. Latent representations are lower-dimensional than token distributions and often encode dense, structured information about semantics and reasoning state. This makes latent regularization a potentially better way for semantic alignment than explicit token-level constraints.

The immediate challenge is that, unlike token probabilities, latent representations do not admit an explicit probability density, making standard divergence measures a computational challenge. To address this, we adopt a technique inspired by GANs~\cite{goodfellow2014generative}: we introduce discriminators that distinguish representations generated by the policy from those generated by the reference model. We show that optimizing the policy against such discriminators is equivalent to minimizing a latent-space divergence, yielding an efficient adversarial regularizer that can be added to existing preference optimization objectives. To fully leverage paired preference signals, we move beyond the standard binary GAN formulation and introduce a quad representation framework. Specifically, we design a contrastive objective by training two discriminators that jointly distinguish high-quality and low-quality representations, while retaining the original offline preference optimization objective. This formulation enables the policy model to learn from pairwise preference dataset while receiving dense structural feedback through latent-space optimization. In this work, we focus on offline preference optimization (OPO)-style methods. Our contributions are summarized as follows:

\begin{itemize}[nosep, leftmargin=15pt]
\item We propose \textbf{GANPO} (generative adversarial network preference optimization), the first {latent-space regularization} for language-model preference optimization, which introduces an efficient {plug-and-play} adversarial regularizer that can be added to existing preference optimization objectives.
\item Experiments across diverse model architectures and tasks demonstrate consistent improvements by plugging GANPO into OPO-style methods on AlpacaEval-2.0.
\item We conduct extensive experimental studies demonstrating that GANPO better preserves the geometry of the model’s internal representation space, yielding improved robustness to stochastic sampling noise and distributional shifts. Moreover, GANPO maintains comparable performance on downstream tasks with modest additional computational overhead.
\end{itemize}

\section{Preliminaries}
% \vspace{-0.1cm}
\subsection{Preference Optimization}
In modern language-model alignment, preference learning is commonly instantiated through reinforcement learning from human feedback (RLHF)~\citep{christiano2017deep,Ouyang2022TrainingLM} by optimizing a KL-regularized reward objective. For a parameter $\beta > 0$, this objective is given by
\begin{align}
    \max_{\pi_\theta}\ r(\pi_\theta)-\beta\cdot \divg_{\mathrm{KL}}(\pi_\theta \, \| \, \pi_{\mathrm{ref}}), \label{eq:rlhf-obj}
\end{align}
where $ r(\pi_\theta):=\mathbb{E}_{x\sim \mathcal{D}, y\sim \pi_\theta(\cdot\mid x)}[r(x, y)]$ denotes the expected reward of the policy over the data distribution $\mathcal D$, and $\divg_{\mathrm{KL}}(\pi_\theta \, \| \, \pi_{\mathrm{ref}}):=\mathbb{E}_{x\sim \mathcal D}[\divg_{\mathrm{KL}}(\pi_\theta(\cdot \| x) \, \| \, \pi_{\mathrm{ref}}(\cdot \| x))]$ is the expected KL between the policy and a reference distribution.

In many practical settings, the available supervision is offline and pairwise: a dataset $\mathcal{D}$ consisting of preference triplets $(x, y_w, y_l)$, representing a prompt $x$, a chosen response $y_w$, and a rejected response $y_l$. Preference optimization methods working with such offline datasets are referred to as offline preference optimization (OPO).

A standard approach of OPO is Direct Preference Optimization (DPO)~\citep{rafailov2023dpo}. 
DPO uses the analytical solution to the KL-regularized RLHF objective (equation~\ref{eq:rlhf-obj}) to remove the explicit reward function. Instead, the reward $r(x,y)$ is implicitly reparameterized in terms of the optimal policy $\pi_\theta$ and the reference policy $\pi_{\text{ref}}$.
% This implicit reward function is given by:
% \begin{equation}
%     r(x, y) = \beta \log \frac{\pi_\theta(y \mid x)}{\pi_{\text{ref}}(y \mid x)} + \beta \log Z(x),
%     \label{eq:dpo_reward}
% \end{equation}
% where $\beta$ controls the strength of the KL divergence constraint and $Z(x)$ represents the partition function. 
To model human preferences, DPO incorporates this formulation into the Bradley-Terry model~\citep{bradley1952rank}, resulting in a tractable objective that directly optimizes the policy:
\begin{align}
    &\mathcal{L}_{\text{DPO}}(\pi_\theta; \pi_{\text{ref}}) = -\mathbb{E}_{(x, y_w, y_l) \sim \mathcal{D}} \\&\left[ \log \sigma \left( \beta \log \frac{\pi_\theta(y_w \mid x)}{\pi_{\text{ref}}(y_w \mid x)} - \beta \log \frac{\pi_\theta(y_l \mid x)}{\pi_{\text{ref}}(y_l \mid x)} \right) \right],
    \label{eq:dpo_loss}
\end{align}
where $\sigma(\cdot)$ stands for the sigmoid function. DPO implicitly incorporates the KL regularization with strength $\beta>0$.
%the preference dataset $\mathcal{D}$ consists of triplets $(x, y_w, y_l)$ representing the prompt, the chosen (winning) response, and the rejected (losing) response, respectively.

\subsection{Generative Adversarial Networks}
Generative Adversarial Networks (GANs)~\citep{goodfellow2014generative} is a framework for learning generative models through an adversarial training process. This framework formulates a minimax, two-player game between two components: a generator $\pi_{G}$ and a discriminator $D$. The generator $\pi_G$ aims to produce samples similar to those from the real distribution $\pi_{\text{real}}$. The discriminator $D(x)\in (0, 1)$ is trained to distinguish between real samples drawn from the data distribution $\pi_{\text{real}}$ and fake samples produced by the generator $\pi_G$.
The adversarial interaction between these two models is formalized as the following minimax optimization problem in standard GANs:
\begin{equation}
    \min_{G} \max_{D}
\mathbb{E}_{ {x} \sim \pi_{\text{real}}}[\log D( {x})]
+ \mathbb{E}_{ x'\sim \pi_G}[\log (1 - D(x'))]. \label{eq:gan-obj}
\end{equation}

Our method integrates these two approaches in a principled way to exploit latent space geometry, as detailed next.

% While effective in continuous domains (e.g., images), applying GANs directly to text domain presents a challenge: the operation involves non-differentiable sampling, preventing the gradients $\nabla_{\theta_g}$ from backpropagating from $D$ to $G$. This limitation motivates our proposed \textsc{GANPO} framework, which shifts the adversarial game to the continuous latent representation space.

% \vspace{-0.3cm}
\section{Latent Adversarial Regularization}\label{sec:method}
% \vspace{-0.1cm}

% \begin{align}
% \max_{\theta}
% ~~
% r(\pi_\theta)
% -
% \beta\,\mathrm{KL}\big(\pi_\theta\, \|\,\pi_{\mathrm{ref}}\big),
% \label{eq:token_kl_reg}
% \end{align}
% where 
%     $\mathrm{KL}\big(\pi_\theta\, \|\,\pi_{\mathrm{ref}}\big):=\mathbb{E}_{x}\Big[\mathrm{KL}\!\big(\pi_\theta(\cdot\mid x)\,\|\,\pi_{\mathrm{ref}}(\cdot\mid x)\big)\Big]$,
% and $x$ is the prompt/context and the KL is over next-token (or sequence) distributions. DPO takes this implicitly to obtain the DPO loss function.

In this section, we introduce latent space adversarial regularization, leading to the proposed method GANPO. 

\subsection{Latent-Space Regularization}\label{sec:method-richer}

Offline preference optimization, e.g., DPO, uses a loss function (equation \ref{eq:dpo_loss}) that implicitly regularizes the deviation of a learned policy
$\pi_\theta$ from a reference policy $\pi_{\mathrm{ref}}$ using a token-space KL divergence of $\divg_{\mathrm{KL}}(\pi_\theta \, \| \, \pi_{\mathrm{ref}})$.
% \noindent To circumvent the non-differentiability of discrete token sampling, we shift the alignment objective from the token space to the \emph{continuous latent space}. Specifically, we designate the generator's final hidden state representations as the proxy for the generated distribution, enabling direct gradient propagation through the discriminator term $D(G(x))$. Beyond resolving the optimization bottleneck, this shift yields a critical secondary benefit: it forces alignment at the \emph{semantic manifold level} rather than the superficial token level. This stands in sharp contrast to DPO. 
However, token-space divergences can be a coarse proxy for the actual behavioral similarity, as they may assign large divergences to semantically similar outputs. Unlike standard reward models that rely on the token verified reward, our discriminator $D$ aligns the global semantic structure by operating on the representation of the entire sequence.

To explore regularization in latent space, we consider prompt-response pairs  $(x, y)$ drawn from a data distribution, e.g., samples from a preference dataset.
Let $\rep_\theta$ be the corresponding latent representation produced by the policy model $\pi_\theta$, and
$\rep_{\mathrm{ref}}$ the corresponding representation from $\pi_{\mathrm{ref}}$.  Specifically, we designate the LLM's final-layer hidden state representations as the latent space representations.
The induced {{{representation}} distributions} are respectively denoted by 
\begin{align}
\rep_\theta\sim \repp_\theta 
\qquad \text{and }\quad 
\rep_{\mathrm{ref}} \sim \repp_{\mathrm{ref}}.
\end{align}
% \andy{don't understand why these are bolded; vectors vs. others is not that clear of a distinction here, no?}

\textbf{{The general form of latent space regularized OPO.}}
Given a divergence function $\divg$ of the two latent representation distributions $\repp_\theta$ and $\repp_{\mathrm{ref}}$, we augment the standard OPO loss in a plug-and-play manner as follows, so the original loss can be from any existing policy optimization methods. 
\begin{align}
\min_{\pi_\theta}\ \mathcal{L}_{\text{OPO}}(\pi_\theta; \pi_{\text{ref}}) + \lambda \cdot \divg(\repp_\theta\, \| \, \repp_{\mathrm{ref}}).
\label{eq:latent_reg}
\end{align}

\subsection{Generative Adversarial Formulation}

While choosing the divergence $\divg$ in equation~\eqref{eq:latent_reg} to be the $\divg_{\mathrm{KL}}$ is conceptually natural, it typically requires densities of
latent representations, which are generally intractable.

Fortunately, some divergences enjoy a variational form. For example, let $\bar \repp:= \tfrac{1}{2}(\repp_\theta+\repp_{\mathrm{ref}})$ denote their average mixture, the Jensen-Shannon divergence (JSD) has a dual formulation~\citep{goodfellow2014generative} expressed as follows.
\begin{align}
&2\divg_{\mathrm{JSD}}\!\left(\repp_\theta \,\|\, \repp_{\mathrm{ref}}\right)
:=
\divg_{\mathrm{KL}}\!\left(\repp_\theta \,\middle\|\, \bar \repp\right)+\divg_{\mathrm{KL}}\!\left(\repp_{\mathrm{ref}} \,\middle\|\, \bar \repp\right)=\log 4 
\\
&+\sup_{D} \left[
\mathbb{E}_{\rep_{\mathrm{ref}}}\big[\log D(\rep_{\mathrm{ref}})\big]  
+
\mathbb{E}_{\rep_{\theta}}\big[\log\big(1-D(\rep_\theta)\big)\big] \right],
\end{align}
where the supremum over $D$ is taken over every function mapping from the space of latent representations to $[0,1]$. 
%\andy{citation for duality.}

This is precisely the objective for the discriminator in the standard GANs (equation \ref{eq:gan-obj}), where $\rep_{\mathrm{ref}}$ corresponds to the real samples and $\rep_\theta$ corresponds to fake  samples.

\textbf{The relativistic average GANs objective.} The standard GANs are notoriously unstable in training.
To improve optimization stability, we adopt the {relativistic average} version of the GANs objective~\citep{jolicoeur2018relativistic}.
It estimates the probability that a real sample is more realistic than the average of the fake samples in the current batch. 

Let $C_\phi$ be the scalar logit from the discriminator parameterized by $\phi$.
Given a batch of samples, define the average baselines as
\begin{align}
m_\theta := \mathbb{E}_{\rep\sim \repp_\theta}[C_\phi(\rep)],
\qquad
m_{\mathrm{ref}} := \mathbb{E}_{\rep\sim \repp_{\mathrm{ref}}}[C_\phi(\rep)].
\end{align}
The relativistic discriminator $\tilde D_\phi$ is defined to be
\begin{align}
\tilde D_\phi(\rep) := \begin{cases}
    \sigma\big(C_\phi(\rep) - m_\theta\big)\quad &\text{if }\  \rep \sim \repp_{\mathrm{ref}},\\
    \sigma\big(C_\phi(\rep) - m_{\mathrm{ref}}\big)\quad &\text{if }\  \rep \sim \repp_{\theta}.
\end{cases}
\label{eq:ragan_D}
\end{align}
where $\sigma$ is the sigmoid function.

The regularization we use in GANPO is therefore implemented by relativistic average GANs (RaGANs), where the divergence is formulated as follows.
\begin{align}
&\divg_{\mathrm{Ra}}(\repp_{\theta} \, \| \,\repp_{\mathrm{ref}})=\log 4 \\
&+ \sup_{\phi}
\Big[
\mathbb{E}_{\rep_{\mathrm{ref}}}\big[\log \tilde D_\phi(\rep_{\mathrm{ref}})\big] +\mathbb{E}_{\rep_{\theta}}\big[\log\big(1-\tilde D_\phi(\rep_\theta)\big)\big] \Big] \label{eq:ragan_loss}.
\end{align}

It is shown in \citet{jolicoeur2020relativistic} that, under the same type of assumptions required for the dual formulation of the $\divg_{\mathrm{JSD}}$, $\divg_{\mathrm{Ra}}$ is a well-defined divergence. I.e., it satisfies: (i) $\divg_{\mathrm{Ra}} (\repp_{\theta} \, \| \,\repp_{\mathrm{ref}}) \geq 0 $; and (ii) $\divg_{Ra}(\repp_{\theta} \, \| \,\repp_{\mathrm{ref}}) =0 \iff \repp_{\theta}=\repp_{\mathrm{ref}}$. A detailed theoretical analysis of this claim is provided in Appendix~\ref{appdix:divg}.  
% \andy{Produce these statements in an appendix.} \andy{What's lost by going to these distributions? It seems to me that I can express every discriminator as $\tilde D$, in which case this would be without loss and just a reparameterization. GOod to be clear.}

To further simplify the divergence and prepare for implementation, we can interpret  $\tilde D_\phi(\rep)\in(0,1)$ to be the discriminator predicting the probability that $\rep$ comes from the reference  distribution, where we assign labels $1$ to $\rep_{\mathrm{ref}}$ and $l=0$ to $\rep_{\theta}$. Then the discriminator objective in equation~\eqref{eq:ragan_loss}
is exactly a binary cross-entropy (BCE) term for this binary classification problem. 

In general, for two generic distributions of representations with random samples $\rep_1$ (label $1$) and $\rep_2$ (label $0$), denote
\begin{align}
&\bce_\phi(\rep_1, \rep_2)\\
&:=
-\mathbb{E}_{\rep_1}\!\left[\log \tilde D_\phi(\rep_1)\right]
-\mathbb{E}_{\rep_2}\!\left[\log\!\left(1-\tilde D_\phi(\rep_2)\right)\right].    
\end{align}
Therefore, the relativistic average divergence (equation~\ref{eq:ragan_loss}) can be simply written as
\begin{align}
    \divg_{\mathrm{Ra}}(\repp_{\theta} \, \| \,\repp_{\mathrm{ref}})&=\log 4 + \sup_\phi \ -\bce_\phi(\rep_{\mathrm{ref}}, \rep_\theta)\\
    &=\log 4 -\inf_\phi \ \bce_\phi(\rep_{\mathrm{ref}}, \rep_\theta). \label{eq:bce-divg}
\end{align}

Up to this point, we have defined the latent-space regularization using the relativistic average divergence $\divg_{\mathrm{Ra}}$. This divergence is reparameterized through the variational formulation in \eqref{eq:bce-divg}, where $\phi$ can be interpreted as a discriminator trained to minimize the $\bce$ loss. This leads to the GAN preference optimization algorithm, as presented next.

\section{GAN Preference Optimization}
Before implementing GANPO, a design choice must be made: which  data  are used to extract latent representations.
Preference alignment is typically implemented with datasets consisting of chosen and rejected pairs $(y_w, y_l)$. 
To fully utilize the paired nature of preference data $(x, y_w, y_l)$, we move beyond the standard binary GAN objective to a quad representation space.

\subsection{GANPO Loss Functions}
% Standard GANs optimize a discriminator $D$ to estimate the absolute probability that a sample is real ($D(x) \to 1$). However, as shown by \citet{jolicoeur2018relativistic}, this formulation is often unstable because the generator improves without necessarily decreasing the discriminator's confidence in real data. Instead, we adopt a Relativistic Average GAN (RaGAN) framework. In our context, this means the discriminator should not just classify a sample as ``Real" or ``Fake," but estimate the probability that the real data is more realistic than the generated data, on average.

% \textbf{The quad representation space.} 
Given samples $(x, y_w, y_l)\sim \mathcal{D}$ from a preference dataset, we implement GANPO on a quad-tuple of latent representations at each training step, defined as follows. The latent representations are the last layer of hidden outputs from the model inferences.
\begin{enumerate}[nosep, leftmargin=15pt]
    \item $\rep_{\mathrm{ref}}^{+}$ from the $\mathrm{ref}$  on the chosen response ($x, y_w$). 
    \item $\rep_{\mathrm{ref}}^{-}$ from the $\mathrm{ref}$  on the chosen response ($x, y_l$). 
    \item $\rep_{\theta}^{+}$ from the $\mathrm{policy}$ on the chosen response ($x, y_w$).
    \item $\rep_{\theta}^{-}$ from the $\mathrm{policy}$ on the chosen response ($x, y_l$). 
\end{enumerate}

% This quad structure allows us to explicitly model both attraction toward desirable features and repulsion from undesirable ones. Concretely, it enables the training of \emph{two} specialized discriminators: a positive discriminator $\phi_{\mathrm{pos}}$ focusing on aligning representations of preferred responses, and a negative discriminator $\phi_{\mathrm{neg}}$ focusing on separation from representations of rejected responses.

Given four representations, we treat the reference model’s positive and negative representations as \emph{anchors} in latent space. To effectively exploit signals from both two models, we adopt a \emph{relativistic average discriminator}, which enables the discriminator to reason about representations in a comparative manner rather than in isolation. Concretely, we introduce two discriminators that respectively model the distributions of ``good'' and ``bad'' latent representations, since two manifolds might be topologically distinct, and a single discriminator cannot separate multiple distributions simultaneously. As formalized in Eq.~\eqref{eq:bce-divg}, both discriminators are trained using relativistic binary cross-entropy objectives, while the policy model is regularized against these discriminators to align its latent representations accordingly. This design allows the model to simultaneously distinguish high- and low-quality responses and to receive dense structural feedback during optimization.

% \yibo{need to be more clear about the meaning of these losses. Ideally connecting both the regularization perspective and GAN perspective} \andy{agreed with Yibo. It's not clear again whether this is another restriction, or a particular instantiation of the dual form for binary cross-entropy above.}

%This is implemented via a running mean of the discriminator scores, ensuring stability and efficiency. Let $\mu = \mathbb{E}_{\rep \sim \pi}[D(\rep)]$ be the average discriminator score of the $G$'s current batch. The relativistic discriminator objective is to maximize:

% \begin{equation}
%     \mathbb{E}_{x_r \sim p_{real}} [\log(\tilde{D}(x_r))] + \mathbb{E}_{x_f \sim p_{fake}} [\log(1 - \tilde{D}(x_f))]
% \end{equation}

% Where the relativistic output $\tilde{D}$ is defined as:

% \begin{equation}
%     \tilde{D}(x_r) = \sigma(C(x_r) - \mathbb{E}_{x_f}[C(x_f)])
% \end{equation}

% (Here $C(x)$ is the raw logit output before sigmoid $\sigma$). In our GANPO implementation, we translate this into two coupled objectives: 

a) \textbf{Generator Optimization}: The policy model $\theta$ minimizes the latent-space regularization $\divg_{\mathrm{Ra}}(\repp^+_{\theta} \, \| \,\repp^+_{\mathrm{ref}})$   and $\divg_{\mathrm{Ra}}(\repp^-_{\theta} \, \| \,\repp^-_{\mathrm{ref}})$.  Given trained discriminators $\phi_{\mathrm{pos}}$ and $\phi_{\mathrm{neg}}$, according to \eqref{eq:bce-divg}, the latent space adversarial regularization loss for the policy model $\theta$ is:
\begin{align}\label{eq:g}
  \mathcal{L}_{\mathrm{adv}} :=  \underbrace{-\bce_{\phi_{\mathrm{pos}}}( \rep_{\mathrm{ref}}^+ , \rep_{\theta}^+ ) }_{\text{Mimic Good}}  \underbrace{-\bce_{\phi_{\mathrm{neg}}}( \rep_{\mathrm{ref}}^-, \rep_{\theta}^-)}_{\text{Align Bad}}
\end{align}
% With the discriminators trained, the latent adversarial loss serves as a regularization term $\divg(\repp_{\theta} \, \| \,\repp_{\mathrm{ref}})$ that pushes the policy representation close to the reference representation. 
We note that in the actual implementation, we view the moving average $m_\theta$ as a constant. Thus, only $\mathbb{E}_{\rep_{\theta}}\big[\log\big(1-\tilde D_\phi(\rep_\theta)\big)\big]$ in the $\bce$ loss is relevant  for the optimization of the policy model $\theta$.

\begin{algorithm}[tb]
\caption{GAN Preference Optimization (GANPO)}
\label{alg:quadgan}
\begin{algorithmic}[1]
\STATE {\bfseries Input:} Preference dataset $\mathcal{D} = \{(x, y_w, y_l)\}$, Policy $\pi_\theta$, Reference $\pi_{\text{ref}}$, Discriminators ${\phi_{\mathrm{pos}}}$ and ${\phi_{\mathrm{neg}}}$.
\STATE {\bfseries Hyperparameters:} Learning rate $\eta$, Adversarial weight $\lambda$, moving average decay rate $\alpha$.
\STATE {\bfseries Initialize:} Global running means $\mu_{\text{pos}} \leftarrow 0$, $\mu_{\text{neg}} \leftarrow 0$.

\FOR{each training step $t=1, \dots, T$}
    \STATE \textbf{1. Data Sampling}
    \STATE Sample batch $\mathcal{B} = \{(x, y_w, y_l)\} \sim \mathcal{D}$.
    
    \STATE \textbf{2. Feature Extraction (Latent Space)}
    \STATE Get last hidden states from Policy $\pi_\theta$ (with gradients) and Reference $\pi_{\text{ref}}$ (frozen):
    \STATE \quad $\rep_{\theta}^+, \rep_{\theta}^- \leftarrow \text{Forward}(\pi_\theta, \mathcal{B})$
    \STATE \quad $\rep_{\mathrm{ref}}^+, \rep_{\mathrm{ref}}^- \leftarrow \text{Forward}(\pi_{\text{ref}}, \mathcal{B})$

    \STATE \textbf{3. Discriminator Optimization (Relativistic)}
    \STATE Compute raw logits $s = C_\phi(\rep)$ for all four $\rep$.
    \STATE Update global running means via moving average:
    \STATE \quad $\mu_{\text{pos}} \leftarrow \alpha \mu_{\text{pos}} + (1-\alpha) \text{Mean}(s_{\text{ref}}^+)$
    \STATE \quad $\mu_{\text{neg}} \leftarrow \alpha \mu_{\text{neg}} + (1-\alpha) \text{Mean}(s_{\text{ref}}^-)$
    
    \STATE Compute Discriminators losses $\mathcal{L}_{\phi_{\mathrm{pos}}}$ and $\mathcal{L}_{\phi_{\mathrm{neg}}}$ with Eq~\ref{eq:d-pos} \& Eq~\ref{eq:d-neg}
    
    \STATE Update Discriminators: 
    \STATE \quad $\phi_{\mathrm{pos}} \leftarrow \phi_{\mathrm{pos}} - \eta \nabla_{\phi_{\mathrm{pos}}} (\mathcal{L}_{\phi_{\mathrm{pos}}})$
    \STATE \quad $\phi_{\mathrm{neg}} \leftarrow \phi_{\mathrm{neg}} - \eta \nabla_{\phi_{\mathrm{neg}}} (\mathcal{L}_{\phi_{\mathrm{neg}}})$

    \STATE \textbf{4. Generator Optimization}
    \STATE Compute Offline PO (e.g., DPO) Loss $\mathcal{L}_{\mathrm{OPO}}$.
    \STATE Compute Generator Loss $\mathcal{L}_{\operatorname{adv}}$ using Eq.~\ref{eq:g}
    
    \STATE Update Generator: 
    \STATE \quad $\theta \leftarrow \theta - \eta \nabla_\theta (\mathcal{L}_{\mathrm{OPO}} + \lambda \mathcal{L}_{\mathrm{adv}})$
\ENDFOR
\end{algorithmic}
\end{algorithm}

b) \textbf{Positive Discriminator ($\phi_{\mathrm{pos}}$)}.
The positive discriminator minimizes $\bce_{\phi_{\mathrm{pos}}}(\rep_{\mathrm{ref}}^+, \rep_{\theta}^+)$, i.e., it aims to give the reference good representation to receive a higher score than the policy's good representation. To better utilize the latent space geometry, we ask it to also distinguish between the policy's good representation and the reference bad representation, i.e., minimizing $\bce_{\phi_{\mathrm{pos}}}(\rep_{\theta}^+ , \rep_{\mathrm{ref}}^-)$. This design allows the discriminator to capture fine-grained preference structure beyond simple real-fake classification. We have 
\begin{align}\label{eq:d-pos}
\mathcal{L}_{\phi_{\mathrm{pos}}}:=\underbrace{\bce_{\phi_{\mathrm{pos}}}(\rep_{\mathrm{ref}}^+, \rep_{\theta}^+)}_{\text{Ref Good $>$ Policy Good}} + \underbrace{\bce_{\phi_{\mathrm{pos}}}(\rep_{\theta}^+ , \rep_{\mathrm{ref}}^-)}_{\text{Policy Good $>$ Ref Bad}}
\end{align}

c) \textbf{The Negative Discriminator ($\phi_{\mathrm{neg}}$)}: Similarly, we have 

\begin{align}\label{eq:d-neg}
 \mathcal{L}_{\phi_{\mathrm{neg}}}:=\underbrace{\bce_{\phi_{\mathrm{neg}}}(\rep_{\mathrm{ref}}^-, \rep_{\theta}^-))}_{\text{Ref Bad $>$ Policy Bad}} + \underbrace{\bce_{\phi_{\mathrm{neg}}}(\rep_{\theta}^-, \rep_{\mathrm{ref}}^+)}_{\text{Policy Bad $>$ Ref Good}}
\end{align}

Thus, the generator (policy model) and the discriminators are optimized alternately, as detailed in Algorithm~\ref{alg:quadgan}.

\subsection{Design Choices: from a GAN Perspectives} 
Having motivated GANPO from a regularization perspective, we now discuss it from a purely GAN-based viewpoint.

\textbf{The structure-preserving adversarial game.} Viewing preference alignment through a GAN-style lens highlights fundamental limitations of offline optimization methods such as DPO. Because DPO operates solely on a fixed preference dataset, the policy is trained to separate preferred and rejected responses within the data distribution. This disconnect encourages spurious correlations, most notably between implicit reward and response length, leading to verbosity rather than genuine semantic improvement~\citep{liu2024length}. GANPO mitigates this issue by introducing an adversarial discriminator that operates directly on latent representations, providing dense structural feedback while framing preference alignment as a zero-sum game between the generator and the discriminator, in which both components are jointly strengthened through adversarial optimization. This adversarial signal acts as a geometry-preserving regularizer, constraining the policy to remain aligned with the reference manifold of high-quality responses even under distributional shift. We discuss more on this in Section~\ref{sec:results}.

\textbf{Design choice: the definition of ``real'' data.} 
From a GAN perspective, the ``real” data in GANPO consists of representations generated by the reference model, which the policy model aims to match. Alternatively, one might ask what would happen if the “real” data were representations obtained from a stronger external teacher model. While $\pi_{ref}$ may be sub-optimal in text generation, its latent manifold represents the well-formed structure of natural language acquired during pre-training. We view the adversarial loss not as a ``correctness'' objective (handled by DPO), but as a ``syntax/manifold'' constraint to prevent mode-collapse. Further, we argue that anchoring to the reference model offers two key advantages over anchoring to a teacher model.

\textbf{a) Manifold consistency for training stability.} A strong teacher model often lies on a distributional manifold that is too dissimilar from the policy, causing the discriminator to learn superficial stylistic differences rather than meaningful structural distinctions, which leads to rapid saturation. With reference-anchored training, we ensure meaningful distributional overlap, forcing the discriminator to learn semantic distinctions and provide dense, informative gradients. 

\textbf{b) Computational efficiency.} Sampling from an external Teacher at each training step is prohibitively expensive. In contrast, $\pi_{\mathrm{ref}}$ is usually required for preference optimization (e.g., DPO), enabling a fully offline, self-contained adversarial training loop with small additional overhead.

\section{Experiment}
% \vspace{-0.1cm}
\begin{table*}[h]
\centering
\small
\setlength{\tabcolsep}{10.5pt}
\begin{tabularx}{\textwidth}{l c c c c c c c c}
\toprule
 & \multicolumn{4}{c}{\textbf{Gemma2-2B-it}} & \multicolumn{4}{c}{\textbf{Llama3-8B-Instruct}} \\
\cmidrule(lr){2-5} \cmidrule(lr){6-9}

\textbf{Method}
& \textbf{Disc.}
& \textbf{Win}
& \textbf{LC-Win}
& \textbf{Len}
& \textbf{Disc.}
& \textbf{Win}
& \textbf{LC-Win}
& \textbf{Len} \\

\midrule
\bf DPO
& N/A & 22.76 & 27.79 & 1668
& N/A & 33.90 & 32.34 & 2041 \\

\bf GANPO (DPO)
& Transformer & \textbf{24.17} & \textbf{29.69} & 1664
& Transformer & \textbf{35.23} & \textbf{33.87} & 2043 \\
\midrule
\bf SimPO
& N/A & 30.66 & 36.03 & 1740
& N/A & 44.09 & 48.31 & 1836 \\

\bf GANPO (SimPO)
& Transformer & \textbf{31.37} & \textbf{36.74} & 1745
& Transformer & \textbf{46.11} & \textbf{50.48} & 1834 \\

\bottomrule
\end{tabularx}
\caption{\textbf{AlpacaEval 2.0 (\texttt{weighted\_alpaca\_eval\_gpt4\_turbo}) results}. GANPO yields consistent gains in raw and length-controlled win rates over DPO and SimPO across model scales, without increasing response length.}
\label{tab:alpaca_eval}
% \vspace{-0.25in}
\end{table*}

\begin{figure*}[t]
    \centering
    \begin{subfigure}[t]{0.48\linewidth}
        \centering
        \includegraphics[width=0.92\linewidth]{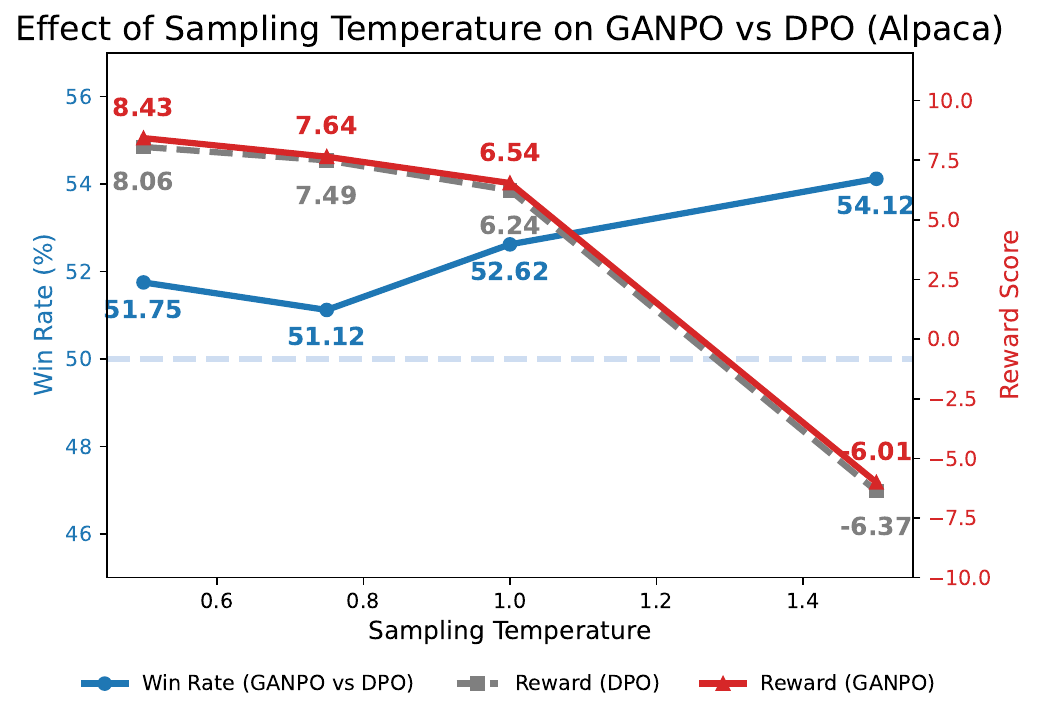}
        \caption{\textbf{AlpacaEval.} GANPO widens the performance gap over DPO as entropy increases ($T \geq 1.0$), demonstrating better quality retention under stochastic sampling.}
        \label{fig:temp_alpaca}
    \end{subfigure}
    \hfill
    \begin{subfigure}[t]{0.48\linewidth}
        \centering
        \includegraphics[width=0.92\linewidth]{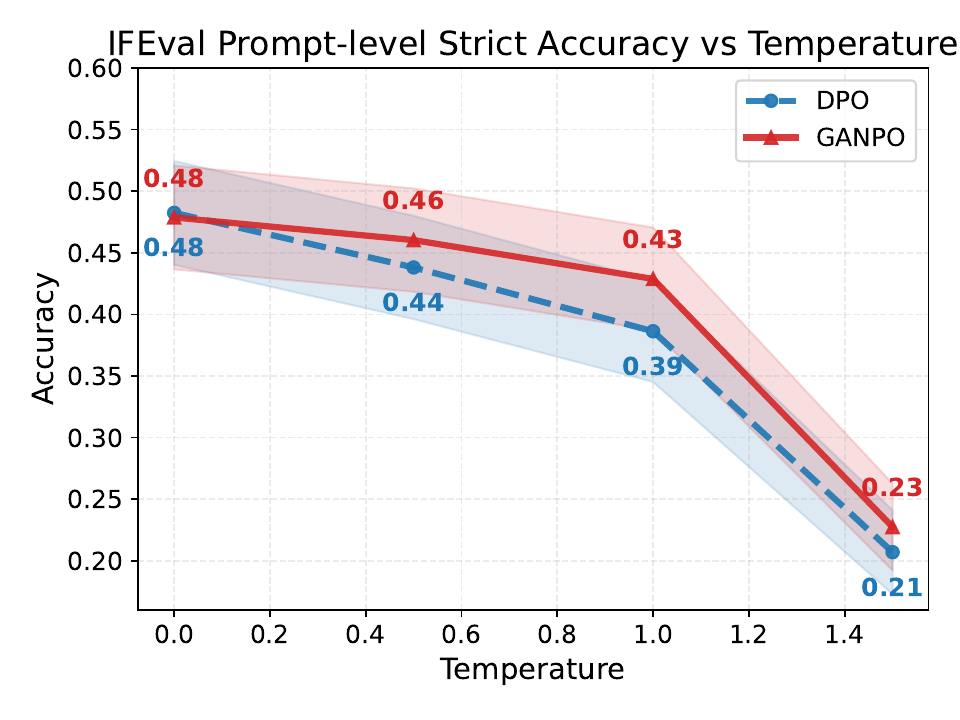}
        \caption{\textbf{IFEval Strict Accuracy.} While DPO suffers from rapid structural degradation as temperature rises, GANPO exhibits resilience, maintaining high instruction adherence even in high-noise regimes.}
        \label{fig:temp_ifeval}
    \end{subfigure}

    \caption{\textbf{Robustness against entropy.} Comparison of model performance across varying sampling temperatures ($T \in [0.0, 1.5]$). Unlike DPO, which relies heavily on greedy decoding for peak performance and collapses under noise, GANPO acts as a structural regularizer, effectively preserving both preference alignment and constraint satisfaction during high-entropy generation.}
    \label{fig:temperature_analysis}
\end{figure*}

\subsection{Experimental setup}
\textbf{Datasets and setup.} Our models are trained on the UltraFeedback dataset~\citep{Cui2024UltraFeedbackBL}, a large-scale preference dataset for instruction-following and dialogue alignment. We follow standard preference optimization training protocols, with minimal modifications to support the GANPO framework. Our evaluation focuses on four aspects: (1) Assessing GANPO’s effectiveness on general instruction-following using AlpacaEval-style metrics; (2) Examining its robustness across model architectures and scales, with experiments on Gemma2-2B-it~\citep{team2024gemma} and Llama-3-8B-Instruct~\citep{llama3modelcard} models; (3) Providing analysis and insights of the structural regularization of GANPO; and (4) Evaluating whether GANPO preserves or improves performance on downstream tasks beyond preference alignment. Full experimental details are provided in Appendix~\ref{app:hyper}. 

\textbf{Baselines.} We compare against DPO~\citep{rafailov2023dpo}, which aligns models via a contrastive objective relative to a fixed reference policy, and SimPO~\citep{meng2024simpo}, which removes the reference model for a simpler and more efficient objective. GANPO is a plug-and-play extension to both methods, retaining their original preference losses while adding structural supervision in the latent space, requiring no changes to the underlying training pipeline.

\subsection{Results and Analysis}\label{sec:results}

\textbf{Preference alignment on open-ended instructions.} Table \ref{tab:alpaca_eval} shows that GANPO consistently improves over its non-adversarial counterparts across both model scales on AlpacaEval-2.0. On Gemma2-2B-it, GANPO yields a clear gain in both raw and length-controlled win rates over DPO (+1.41\% LC-Win) and SimPO (+0.71\% LC-Win), while maintaining comparable response lengths. Similar trends hold for Llama3-8B-Instruct, where GANPO improves LC-Win rates around 1.5\%-2.0\% over DPO and SimPO. These results indicate that adversarial regularization provides benefits to preference optimization, improving alignment quality without relying on increased verbosity.

\textbf{Structural regularization under stochastic decoding.} To examine the impact of the structured regularization imposed by GANPO under increasingly stochastic decoding, we stress-test the Gemma2-2B-it model across a spectrum of sampling temperatures ($T \in [0.0, 1.5]$). Higher temperatures induce greater diversity but simultaneously amplify exposure bias and structural instability, serving as a proxy for out-of-distribution robustness. Our evaluation spans two distinct regimes: 

\textbf{(1) AlpacaEval:} We evaluate general response quality using the Skywork-Reward-V2-Llama-3-8B-it model~\citep{liu2025skywork} as an oracle judge. As shown in Figure~\ref{fig:temp_alpaca}, GANPO consistently achieves higher winrate and reward scores than DPO across a wide range of temperatures. Crucially, the winrate gap widens in high-entropy regimes ($T \ge 1.0$). This divergence indicates that GANPO is more robust under high-entropy generation, where DPO’s token-level optimization becomes increasingly brittle. 

\textbf{(2) IFEval~\citep{zhou2023instructionfollowing}:} To assess the stability of instruction following under noise, we measure the strict prompt-level accuracy with structured outputs on IFEval (Figure~\ref{fig:temp_ifeval}). Here, the contrast is stark: DPO suffers from more severe structural collapse as temperature increases (dropping nearly 20\% in accuracy from $T=0.0$ to $T=1.0$), indicating that its adherence to constraints heavily relies on greedy decoding. In contrast, GANPO demonstrates stronger resilience, retaining good strict accuracy under stochastic sampling. 

Together, GANPO moves beyond surface-level alignment and learns a structurally robust manifold. As a result, preference alignment and constraint adherence remain stable even when generation trajectories deviate from the optimal path, where purely likelihood-based methods tend to degrade.

\textbf{The effectiveness of $D$.} To evaluate whether the trained discriminator can reliably distinguish high- and low-quality representations, we train a Gemma2-2B-it reward model on the UltraFeedback dataset as a proxy reward and compare its scores against the gold standard Skywork-Reward-V2-Llama-3-8B-it model~\citep{liu2025skywork}. We conduct stress tests under high-entropy generation by sampling a single response ($N=1$) from a large candidate pool ($k=1024$) at elevated temperatures ($T=1.5$ and $T=2.0$). In Figure~\ref{fig:disc_effect}, under these out-of-distribution conditions, the learned reward model exhibits severe reward hacking, collapsing to weak ($r=0.14$) or even negative correlations ($r=-0.50$) with the oracle. In contrast, the discriminator maintains a strong positive correlation ($r=0.59$ and $r=0.52$), demonstrating robustness to distributional shift. These results suggest that the discriminator acts as an effective structural regularizer in latent space, capturing semantic properties rather than surface-level token patterns.

\begin{figure*}[t]
    \centering
    \begin{subfigure}[t]{0.48\linewidth}
        \centering
        \includegraphics[width=0.9\linewidth]{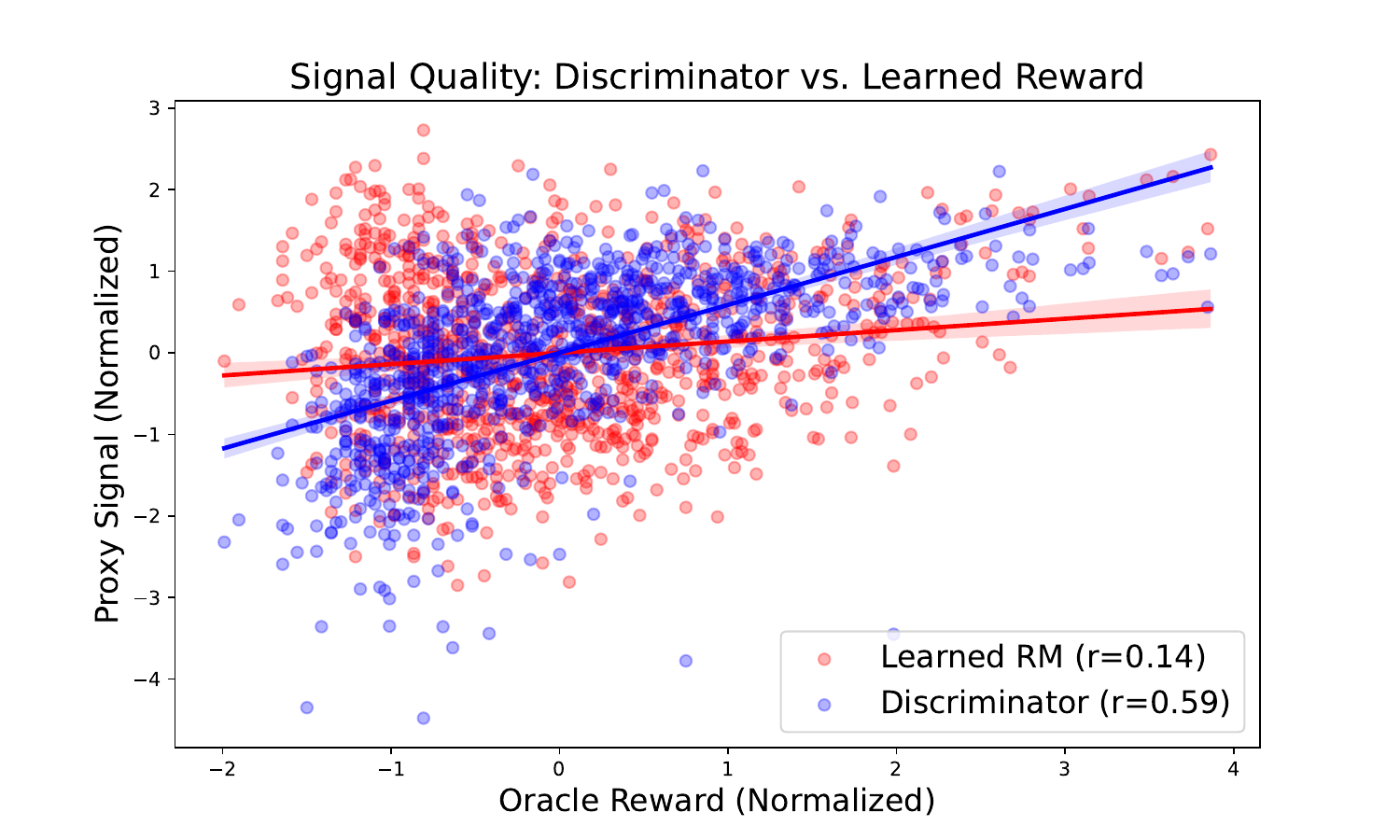}
        \caption{$T=1.5$}
        \label{fig:disc_1.5}
    \end{subfigure}
    \hfill
    \begin{subfigure}[t]{0.48\linewidth}
        \centering
        \includegraphics[width=0.9\linewidth]{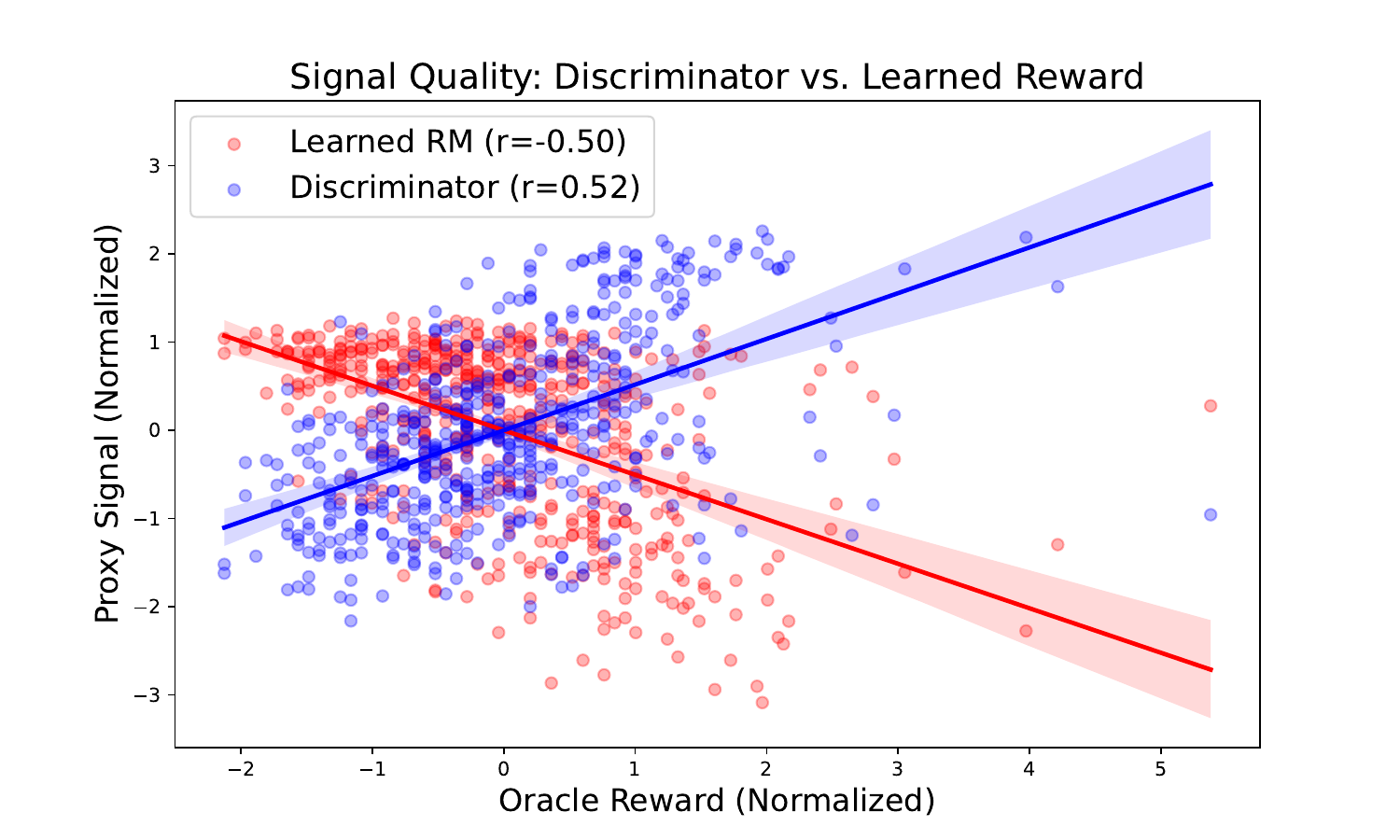}
        \caption{$T=2.0$}
        \label{fig:disc_2}
    \end{subfigure}

    \caption{Comparison of discriminator-based scoring and learned reward models under high-entropy generation. At elevated sampling temperatures ($T=1.5$ and $T=2.0$), the learned reward model exhibits severe reward hacking, including correlation collapse and inversion with respect to the oracle. In contrast, the discriminator maintains a strong positive correlation, demonstrating robustness to out-of-distribution generations and providing stable structural supervision in latent space.}
    \label{fig:disc_effect}
    % \vspace{-0.2in}
\end{figure*}

% if have time, extend to other datasets
\textbf{Downstream evaluation on multiple benchmarks.} Beyond AlpacaEval, GANPO does not degrade, and in several cases improves performance on downstream benchmarks, including math, reasoning, and factuality tasks (Table~\ref{tab:gemma_downstream}). This suggests that the adversarial objective does not overfit to the preference dataset by sacrificing the performance on other tasks. Instead, it acts as a form of structured regularization, encouraging representations that generalize beyond the alignment setting. 

\begin{table}[h]
\centering
\fontsize{7.7}{9}\selectfont
\begin{tabular}{lcccc}
\toprule
\textbf{Method} & \textbf{GSM8K} & \textbf{MMLU} & \textbf{ANLI} (R3) & \textbf{TruthfulQA} \\
                & \textbf{(Math)} & \textbf{(Knowledge)} & \textbf{(Reasoning)} & \textbf{(Factuality)} \\
\midrule
Base        & 46.32 & 56.73 & 47.75 & 53.11 \\
DPO             & 48.37 & \bf57.02 & 47.92 & 55.28 \\
GANPO           & \textbf{48.67} & 56.93 & \bf48.25 & \textbf{55.67} \\
\bottomrule
\end{tabular}
\caption{Downstream evaluation on Gemma2-2B-it model.}\label{tab:gemma_downstream}
\vspace{-0.2in}
\end{table}

\textbf{Architectures of $D$.} Across all experiments, we observe that a Transformer-based discriminator consistently outperforms simpler alternatives, such as fixed MSE critics or shallow MLPs, as shown in Table~\ref{tab:disc_ablation}. This architectural choice enables the discriminator to provide holistic, sequence-level feedback, capturing long-range dependencies and structural properties of text that scalar or local critics fail to model. As reflected in the AlpacaEval results, stronger discriminators translate directly into higher win rates, underscoring the importance of expressive discriminator architectures for effective GAN preference optimization.

\begin{table}[h]
\centering
\fontsize{8}{9}\selectfont
\begin{tabular}{lccc}
\toprule
\textbf{Discriminator} & \textbf{Win Rate} & \textbf{LC-Win Rate} & \textbf{Avg Length} \\
\midrule
N/A                 & 22.76 & 27.79 & 1668 \\
MSE (fixed $D$)     & 22.96 & 27.03 & 1684 \\
MLP                 & 22.66 & 27.52 & 1671 \\
Transformer         & \textbf{24.17} & \textbf{29.69} & 1664 \\
\bottomrule
\end{tabular}
\caption{Effect of $D$'s architecture on alignment performance.}\label{tab:disc_ablation}
\vspace{-0.25in}
\end{table}

\subsection{Discussion and Limitations}

\textbf{Limitations.} Unlike DPO and SimPO, which are parameter-efficient and essentially ``reward-free" in architecture, GANPO requires maintaining and updating a discriminator alongside the policy. The computational overhead is illustrated in Table~\ref{tab:training_cost}. The adversarial game introduces additional complexity in hyperparameter tuning compared to the stability of supervised objectives with modest additional cost. Also, our \textit{reference-anchored training} stabilizes the discriminator by using the SFT model to define the target manifold. This effectively bounds the exploration space. If the SFT model is fundamentally misaligned or possesses a defective latent structure, GANPO may struggle to diverge sufficiently to find a globally optimal policy, effectively inheriting the topological flaws of its anchor. Lastly, we acknowledge that latent space and token space regularization may be complementary and require more investigation.

\textbf{Future Work.} Standard alignment methods often struggle with strict syntactic constraints (e.g., valid JSON, compilable code). Future work should explore augmenting the discriminator with symbolic feedback, injecting compiler signals or logical verifiers directly into the latent loss, to enforce syntax as a differentiable manifold constraint. Also, GANPO currently operates in an offline setting. Extending this to an online ``Self-Play" framework where the model generates its own rollouts to be critiqued by an evolving discriminator could bridge the gap between offline efficiency and the performance benefits of online methods like PPO. Further, since GANPO operates on the representation space rather than discrete tokens, it is inherently modality-agnostic. Adapting this framework to Vision-Language Models (VLMs) could provide a powerful method for aligning cross-modal generation, where structural consistency between text and image representations is critical.

% \vspace{-0.2cm}

\section{Related Work}
\textbf{Preference optimization objectives.} Online preference optimization methods are often complex and difficult to stabilize~\citep{zheng2023secrets, santacroce2023efficient}, motivating the development of simpler and more efficient offline alternatives. DPO~\citep{Rafailov2023DirectPO} is a representative approach, but its lack of an explicit reward model limits its ability to leverage samples from the optimal policy. Prior work addresses this by augmenting preference data using SFT-generated responses or rejection sampling~\citep{Zhao2023SLiCHFSL, liu2024statistical}, and by extending DPO to iterative or self-improving training schemes~\citep{dong2024rlhf, Kim2024sDPODU, Rosset2024DirectNO, xiong2024iterative, yuan2024self}. For latent-space optimization,~\citet{hao2025training, zhu2025reasoning} show that latent representations can lead to improved reasoning capabilities. In this work, we focus on offline alignment with latent-space regularization. We compare GANPO against DPO~\citep{Rafailov2023DirectPO} and SimPO~\citep{meng2024simpo}, showing that GANPO consistently outperforms both with modest additional computational cost.

\textbf{GANs.} GANs formulate learning as a minimax game between a generator and a discriminator~\citep{goodfellow2014generative}, and GAN variants have been extensively studied for stabilizing adversarial distribution matching and improving training dynamics~\citep{zhu2017cyclegan, gulrajani2017wgan, jolicoeur2018relativistic, jolicoeur2020relativistic}. In natural language generation, GAN-based methods have been explored empirically~\citep{zhang2016generating, zhang2017adversarial}, where a discriminator distinguishes generated text from human-written samples, and TextGAIL~\citep{wu2021textgail} adapts adversarial imitation learning to optimize language models as response policies. More recently, minimax formulations have been proposed for preference learning, such as the Adversarial Preference Optimization framework, in which the LLM and the reward model update alternatively via a minmax game in an online manner~\citep{cheng2023adversarial}. Orthogonal to previous work, we introduce a GAN-style adversarial regularizer operating in latent space, designed to complement \emph{offline} preference optimization and mitigate exposure bias and structural degradation in LLM generation.

\paragraph{Reinforcement learning from human feedback.} Reinforcement Learning from Human Feedback (RLHF) is a widely adopted paradigm for aligning large language models (LLMs) with human preferences and values~\citep{christiano2017deep, ziegler2019fine, Ouyang2022TrainingLM, bai2022training}. A classical RLHF pipeline typically consists of three stages: supervised fine-tuning~\citep{zhou2024lima, taori2023stanford, geng2023koala, DatabricksBlog2023DollyV2, kopf2024openassistant, Ding2023EnhancingCL, wang2024openchat, chen2024alpagasus, xia2024less}, reward model training~\citep{gao2023scaling, luo2023wizardmath, chen2024odin, lightman2023let, havrilla2024glore, lambert2024rewardbench}, and policy optimization, most commonly via Proximal Policy Optimization~\citep{schulman2017proximal, anthony2017thinking}. Also, recent work has highlighted systemic challenges throughout the RLHF pipeline~\citep{casper2023open}. Moreover, RLHF has been shown to induce unintended biases, such as excessive verbosity and length-based reward hacking~\citep{dubois2024length, singhal2023long, wang2023far}. In contrast, GANPO is a fully offline alignment method orthogonal to RLHF, as it requires no online rollouts or reinforcement learning objectives. By introducing adversarial structural feedback in the latent space, GANPO shows great potential in future LLM alignment.

\section{Conclusion}
We proposed GANPO, a framework that augments preference learning with adversarial regularization to address the structural degradation inherent in offline methods like DPO. By leveraging Latent-Space Alignment and a Dual-Contrastive Objective, GANPO enables the discriminator to provide differentiable feedback that guides the policy toward high-quality, structurally sound modes. Furthermore, our Reference-Anchored Training ensures that this process remains stable and computationally efficient. Our experiments demonstrate that GANPO significantly outperforms state-of-the-art baselines like DPO and SimPO, particularly in controlling verbosity and maintaining structural coherence. This work validates the hypothesis that adversarial feedback, when applied to latent representations, serves as a crucial regularizer for aligning LLMs with human preferences.

\section{Impact Statement}
This work advances preference optimization methods for large language models. The proposed approach improves the performance and robustness of offline alignment by introducing adversarial structural regularization. We do not anticipate societal or ethical impacts beyond those commonly associated with large language model training and alignment.

\section{Acknowledgments}
SK acknowledges support by NSF 2046795 and 2205329, IES R305C240046, ARPA-H, the MacArthur Foundation, Schmidt Sciences, HAI, OpenAI, Microsoft, and Google. This research used the DeltaAI advanced computing and data resource, which is supported by the National Science Foundation (award OAC 2320345) and the State of Illinois. DeltaAI is a joint effort of the University of Illinois Urbana-Champaign and its National Center for Supercomputing Applications. Further, we gratefully acknowledge the use of GPU computing resources provided by the CAIS Compute Cluster at the Center for AI Safety (Safe.ai).

\nocite{langley00}

\bibliography{references}
\bibliographystyle{icml2026}

%%%%%%%%%%%%%%%%%%%%%%%%%%%%%%%%%%%%%%%%%%%%%%%%%%%%%%%%%%%%%%%%%%%%%%%%%%%%%%%
%%%%%%%%%%%%%%%%%%%%%%%%%%%%%%%%%%%%%%%%%%%%%%%%%%%%%%%%%%%%%%%%%%%%%%%%%%%%%%%
% APPENDIX
%%%%%%%%%%%%%%%%%%%%%%%%%%%%%%%%%%%%%%%%%%%%%%%%%%%%%%%%%%%%%%%%%%%%%%%%%%%%%%%
%%%%%%%%%%%%%%%%%%%%%%%%%%%%%%%%%%%%%%%%%%%%%%%%%%%%%%%%%%%%%%%%%%%%%%%%%%%%%%%
\newpage
\appendix
\onecolumn
\section{Relativistic Average Divergence }\label{appdix:divg}

Consider two probability distributions $p, q$ with support $\mathcal X$. A divergence between two probability distributions is defined as follows.
\begin{definition}[Statistical Divergence]
    Let $\mathcal M$ denote the space of all probability distributions with support $\mathcal X$. A function $\divg:\mathcal M\times \mathcal M \to \mathbb R$ is a divergence if for all $p, q\in \mathcal M$:
    \begin{enumerate}[nosep, leftmargin=15pt]
        \item $\divg(p \, \|\, q)\geq 0$;
        \item $\divg(p \, \|\, q)=0 \iff p=q$.
    \end{enumerate}
\end{definition}

It is shown that the following divergence is well-defined.
\begin{proposition}[Relativistic Average Divergence~\citep{jolicoeur2020relativistic}]\label{prop:ra_divg}
    Let $f:\mathbb R\to \mathbb R$ be a concave function such that $f(0)=0$, $f$ is differentiable at $0$, $f'(0)\neq 0$, $\sup_x f(x)>0$, $\arg\sup_x f(x)>0$. Ley $p, q$ be two distributions with common support $\mathcal X$. Then, 
    \begin{align}
    \divg_{\mathrm{Ra}}^f(p \, \|\, q):= \sup_{C:\mathcal X \to \mathbb R}\ \mathbb{E}_{y\sim q}\left[f\left(C(y)-\mathbb E_{x\sim p} C(x)\right)\right]+\mathbb{E}_{x\sim p}\left[f\left(\mathbb E_{y\sim q} C(y)-C(x)\right)\right]
    \end{align}
    is a divergence.
\end{proposition}
In the main paper (equation~\ref{eq:ragan_loss}), we define the following term.
\begin{align}
&\divg_{\mathrm{Ra}}(\repp_{\theta} \, \| \,\repp_{\mathrm{ref}}) := \sup_{\phi}
\Big[
\mathbb{E}_{\rep_{\mathrm{ref}}}\big[\log \sigma\big(C_\phi(\rep) - m_\theta\big)\big] +\mathbb{E}_{\rep_{\theta}}\big[\log\big(1-\sigma\big(C_\phi(\rep) - m_{\mathrm{ref}}\big)\big)\big] \Big] + \log 4 ,\label{eq:ra-divg-app}
\end{align}
where $C_\phi$ is the scalar logit from the discriminator parameterized by $\phi$; $h\in \mathcal H$ denotes a latent representation; $\sigma$ denotes the sigmoid function;  and the average baselines are
\begin{align}
m_\theta := \mathbb{E}_{\rep\sim \repp_\theta}[C_\phi(\rep)],
\qquad
m_{\mathrm{ref}} := \mathbb{E}_{\rep\sim \repp_{\mathrm{ref}}}[C_\phi(\rep)]. \label{eq:mean}
\end{align}

We can prove that equation~\ref{eq:ra-divg-app} is indeed a well-defined divergence given that $C_\phi$ is taken over all functions mapping from the representation space $\mathcal H\to \mathbb R$.
\begin{proposition}
    The following term is a divergence between $\repp_\theta$ and $\repp_{\mathrm{ref}}$.
    \begin{align}
        \divg_{\mathrm{Ra}}(\repp_{\theta} \, \| \,\repp_{\mathrm{ref}}) := \sup_{C:\mathcal H \to \mathbb R}
\Big[
\mathbb{E}_{\rep_{\mathrm{ref}}}\big[\log \sigma\big(C(\rep) - m_\theta\big)\big] +\mathbb{E}_{\rep_{\theta}}\big[\log\big(1-\sigma\big(C(\rep) - m_{\mathrm{ref}}\big)\big)\big] \Big] + \log 4.
    \end{align}
\end{proposition}
\begin{proof}
    Define function $f:\mathbb{R}\to \mathbb{R}$ as
    \begin{align}
        f(x)=\log \sigma (x) + \log 2.
    \end{align}
    We can see that $\divg_{\mathrm{Ra}}=\divg^f_{\mathrm{Ra}}$ as follows.

    First, it is straightforward to verify that $f$ is strictly concave; everywhere differentiable; $f(0)=-\log 2 + \log 2 = 0$; $f'(0)=\tfrac{1}{2}\neq =0$; $\sup_x f(x)=\log 2>0; \arg\sup_x f(x)=+\infty >0$. Thus, $f$ satisfies the conditions specified by Proposition~\ref{prop:ra_divg}. 
    
    To conclude the proof, we use identity $1-\sigma(x)=\sigma(-x)$, and replace the log-sigmoid function by $f$:
    \begin{align}
         \divg_{\mathrm{Ra}}(\repp_{\theta} \, \| \,\repp_{\mathrm{ref}}) &:=\sup_{C:\mathcal H \to \mathbb R}
\Big[
\mathbb{E}_{\rep_{\mathrm{ref}}}\big[\log \sigma\big(C(\rep) - m_\theta\big)\big] +\mathbb{E}_{\rep_{\theta}}\big[\log\big(1-\sigma\big(C(\rep) - m_{\mathrm{ref}}\big)\big)\big] \Big]+\log 4.\\
&= \sup_{C:\mathcal H \to \mathbb R}
\Big[
\mathbb{E}_{\rep_{\mathrm{ref}}}\big[\log \sigma\big(C(\rep) - m_\theta\big)\big] +\log 2+\mathbb{E}_{\rep_{\theta}}\big[\log\sigma\big(m_{\mathrm{ref}}-C(\rep)\big)\big)\big]+\log 2 \Big]\\
&=\sup_{C:\mathcal H \to \mathbb R}
\Big[
\mathbb{E}_{\rep_{\mathrm{ref}}}\big[f\big(C(\rep) - m_\theta\big)\big] +\mathbb{E}_{\rep_{\theta}}\big[f\big(m_{\mathrm{ref}}-C(\rep)\big)\big)\big] \Big]\\
&= \divg^f_{\mathrm{Ra}}(\repp_{\theta} \, \| \,\repp_{\mathrm{ref}}).
    \end{align}
    Therefore, by Proposition~\ref{prop:ra_divg}, $\divg_{\mathrm{Ra}}(\repp_{\theta} \, \| \,\repp_{\mathrm{ref}})$ is a well-defined divergence.
\end{proof}

\section{Implementation Details and Hyperparameters}\label{app:hyper}
\paragraph{Training details.} For the preference optimization, we use a batch size of 128 and train the models for 1 epoch. Additionally, we set the max sequence length to be 2048 and apply a cosine learning rate schedule with 10\% warmup steps on the preference optimization dataset. As for the learning rates of the generator, for Gemma2-2B-it experiments, we use the learning rate of $5.0e^{-7}$; for Llama-3-8B-Ins experiments, we use the learning rate of $1e^{-6}$. For training the discriminators, we use half of the learning rate of the generator for the discriminator training. Further, we set adversarial
weight $\lambda=1$ and moving average decay rate $\alpha=0.9$ for all experiments. For DPO training, we use $\beta=0.1$ for all experiments; for SimPO training, we follow the setup as in~\citet{meng2024simpo}. For the generation stage, we use a temperature of 0.7 for the Gemma2-2B-it setting and a temperature of 0.9 for Llama3-8B-Instruct settings.

\paragraph{Transformer architecture (Figure~\ref{fig:trans-arc}).} We use a transformer architecture (2 layers for Gemma2-2B-it experiments and 4 layers for Llama3-8B-Instruct experiments). The discriminator operates directly on continuous latent representations produced by the policy model. Input hidden states are first projected to a lower-dimensional space via a spectrally normalized linear layer to stabilize adversarial training. A lightweight Transformer encoder with learned positional embeddings then models global and long-range dependencies across the sequence using pre-layer normalization. Sequence-level representations are obtained via masked mean pooling, ensuring robustness to variable-length inputs. Finally, a spectrally normalized MLP head maps the pooled representation to a scalar score. This design enables the discriminator to capture holistic, structural properties of generation trajectories while remaining computationally efficient and stable.

\paragraph{Computation environment.} All the training experiments in this paper were conducted on 2×H200 or 4xA100 GPUs.

\section{Computational Cost Analysis}
Table~\ref{tab:training_cost} shows that GANPO introduces only modest computational overhead compared to its corresponding DPO and SimPO baselines. On Gemma2-2B-it, GANPO (DPO) increases training time by less than 4\% while using identical hardware, and exhibits similar scaling behavior on Llama-3-8B-Instruct. Although GANPO (SimPO) incurs a larger overhead, it remains within the same GPU budget and does not require additional rollout generation or external teacher queries. Overall, these results demonstrate that adversarial regularization in GANPO can be incorporated into standard preference optimization pipelines with small additional cost, making it a practical and scalable alternative to purely offline alignment methods.

\begin{table}[h]
\centering
\small
\setlength{\tabcolsep}{8pt}
\begin{tabular}{lcccc}
\toprule
\textbf{Method} 
& \multicolumn{2}{c}{\textbf{Gemma2-2B-it}} 
& \multicolumn{2}{c}{\textbf{Llama-3-8B-Instruct}} \\
\cmidrule(lr){2-3} \cmidrule(lr){4-5}
& \textbf{Time} & \textbf{GPU} & \textbf{Time} & \textbf{GPU} \\
\midrule
\bf DPO 
& 2h 31m 50s & 4 $\times$ A100 
& 4h 20m 13s & 4 $\times$ A100 \\

\bf GANPO (DPO) 
& 2h 37m 24s & 4 $\times$ A100 
& 4h 34m 08s & 4 $\times$ A100 \\
\midrule
\bf SimPO 
& 2h 11m 22s & 2 $\times$ H200 
& 3h 10m 40s & 4 $\times$ A100 \\

\bf GANPO (SimPO) 
& 3h 17m 26s & 2 $\times$ H200 
& 4h 15m 52s & 4 $\times$ A100 \\
\bottomrule
\end{tabular}
\caption{Training time and hardware comparison between GANPO and its corresponding DPO/SimPO baselines. GANPO introduces only modest computational overhead while providing consistent alignment improvements.}
\label{tab:training_cost}
\end{table}

\section{Additional Empirical Visualizations}
\subsection{Win-Rate versus Response Length}
Figure~\ref{fig:winrate-vs-length} analyzes win rates across different response length buckets. DPO shows a clear degradation in performance as responses become longer, consistent with prior observations that offline preference optimization tends to exploit length-related artifacts. In contrast, GANPO maintains stable and consistently higher win rates for medium and long responses, suggesting that adversarial structural regularization mitigates verbosity bias and improves preference alignment beyond token-level heuristics.

\begin{figure}
    \centering
    \includegraphics[width=0.65\linewidth]{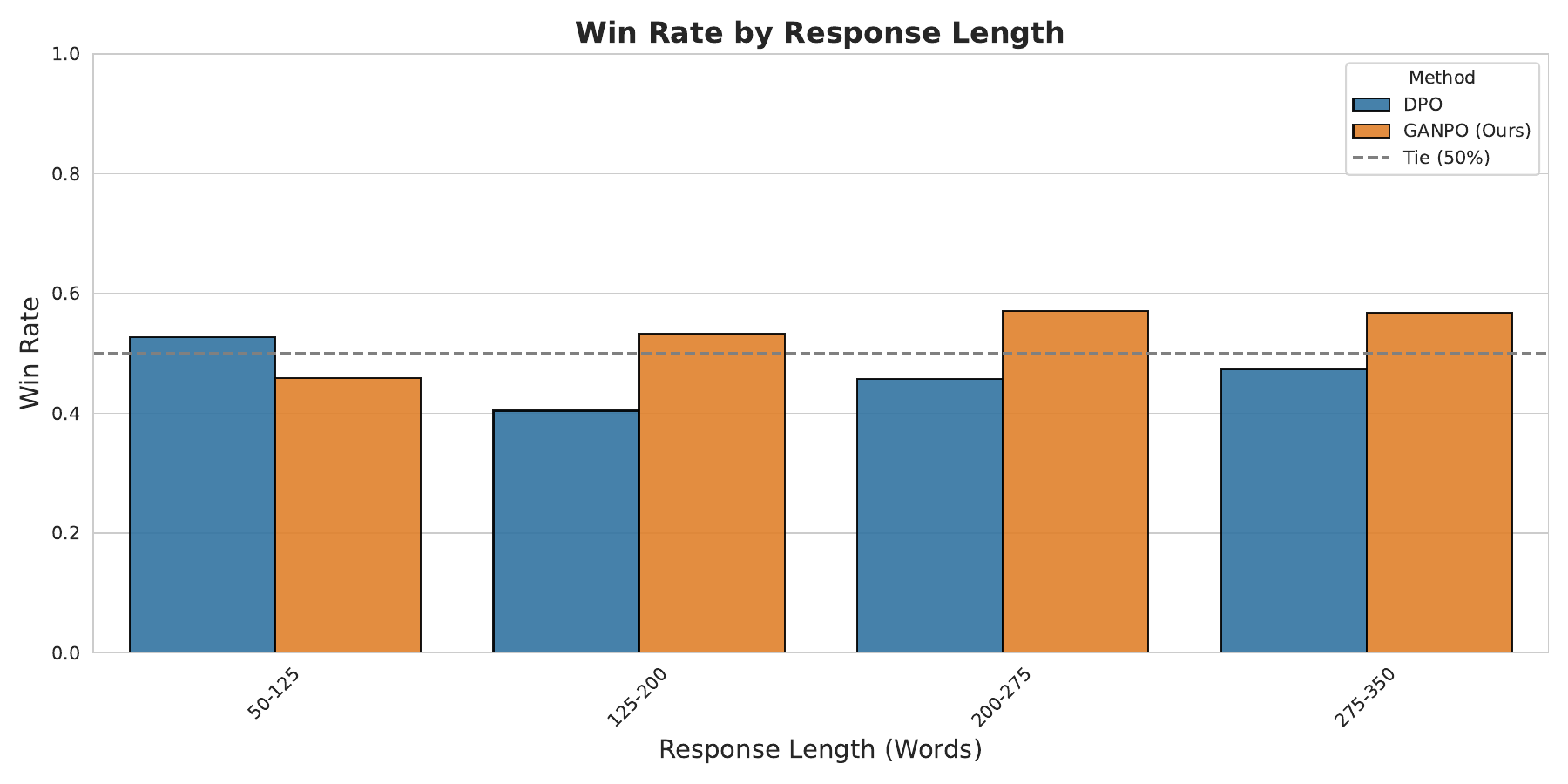}
    \caption{Win rate as a function of response length for DPO and GANPO.}
    \label{fig:winrate-vs-length}
    % \vspace{-0.25in}
\end{figure}

\subsection{Margins Analysis} Across training, GANPO consistently achieves larger preference margins than both DPO and SimPO. We show it happens for both different models and OOP objectives (Figure~\ref{fig:margins_dpo}, Figure~\ref{fig:margins_simpo}, Figure~\ref{fig:margins_dpo_llama}, and Figure~\ref{fig:margins_simpo_llama}), indicating a clearer separation between preferred and rejected responses. These margins increase steadily over optimization, suggesting more stable and effective preference learning dynamics. In contrast to purely likelihood-based objectives, the adversarial component in GANPO provides structured feedback in latent space, which helps reinforce robust preference separation rather than relying on surface-level token correlations. When combined with SimPO, GANPO also shows margin growth, demonstrating that adversarial structural regularization complements existing offline preference objectives by strengthening latent alignment without destabilizing training.

\begin{figure}
    \centering
    \includegraphics[width=0.65\linewidth]{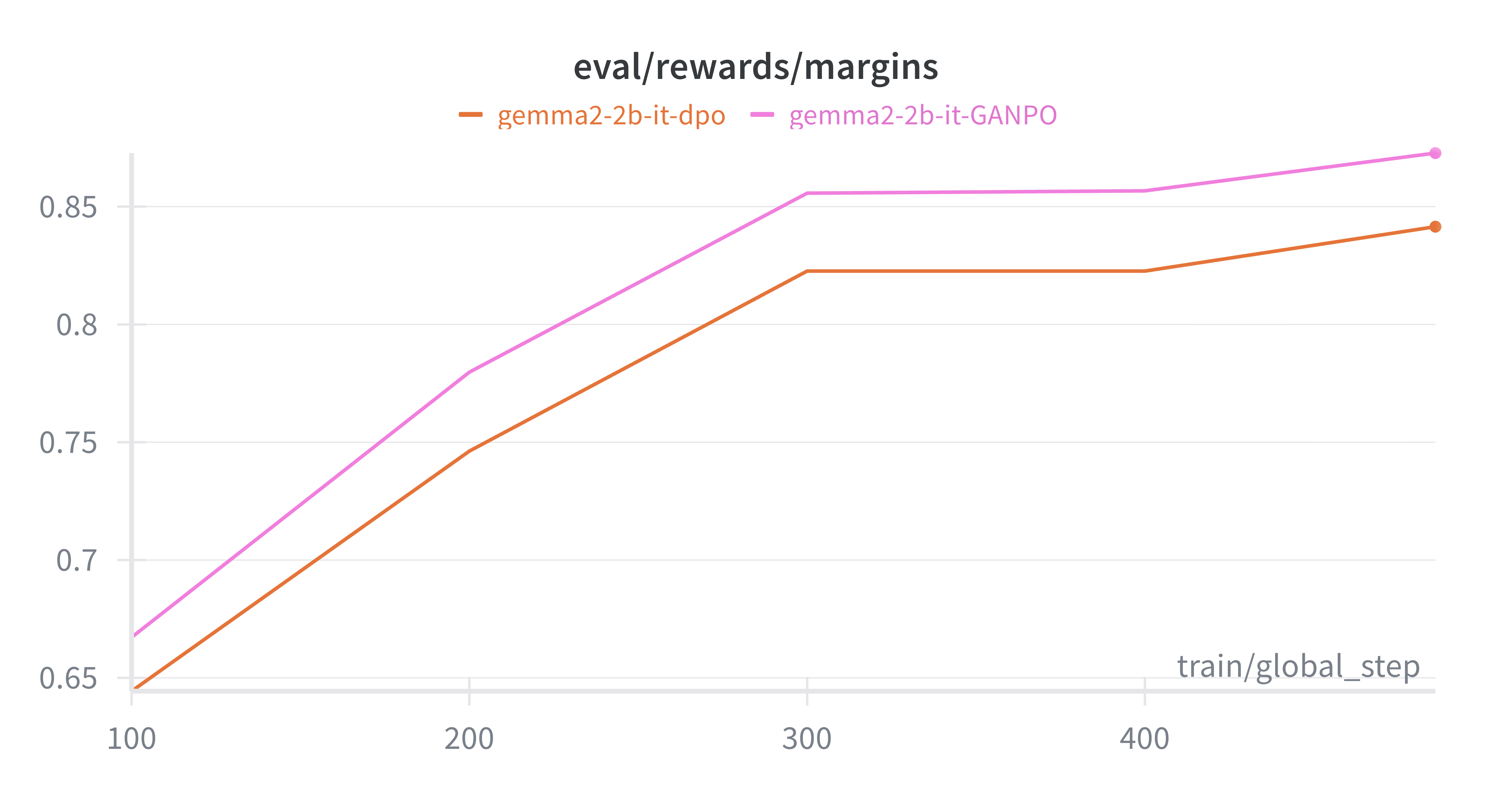}
    \caption{Evolution of reward margins during training on Gemma2-2B-it, comparing DPO and GANPO. }
    \label{fig:margins_dpo}
     % \vspace{-0.25in}
\end{figure}

\begin{figure}
    \centering
    \includegraphics[width=0.65\linewidth]{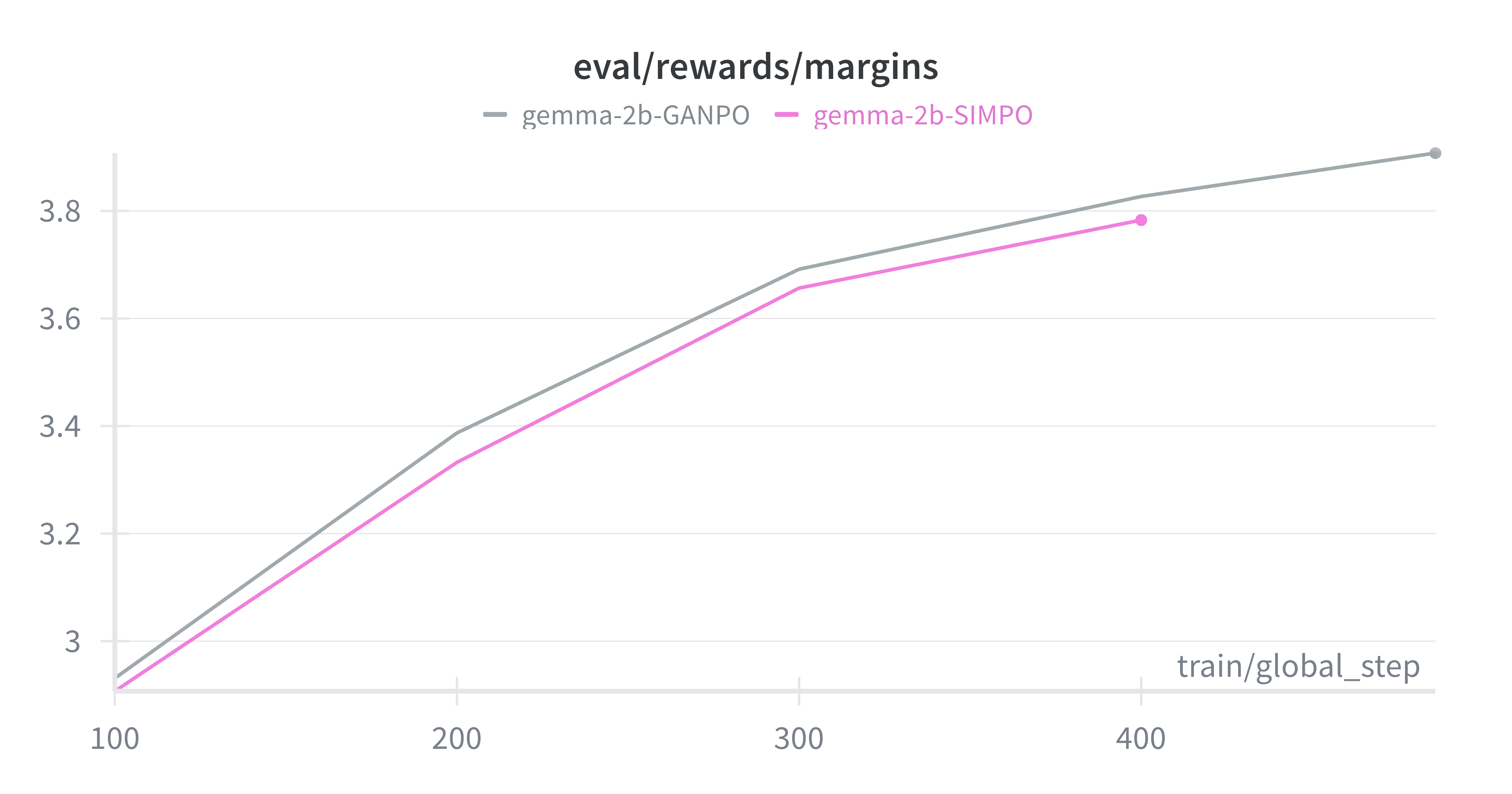}
    \caption{Reward margin comparison between SimPO and GANPO on Gemma2-2B-it. }
    \label{fig:margins_simpo}
     % \vspace{-0.25in}
\end{figure}

\begin{figure}
    \centering
    \includegraphics[width=0.65\linewidth]{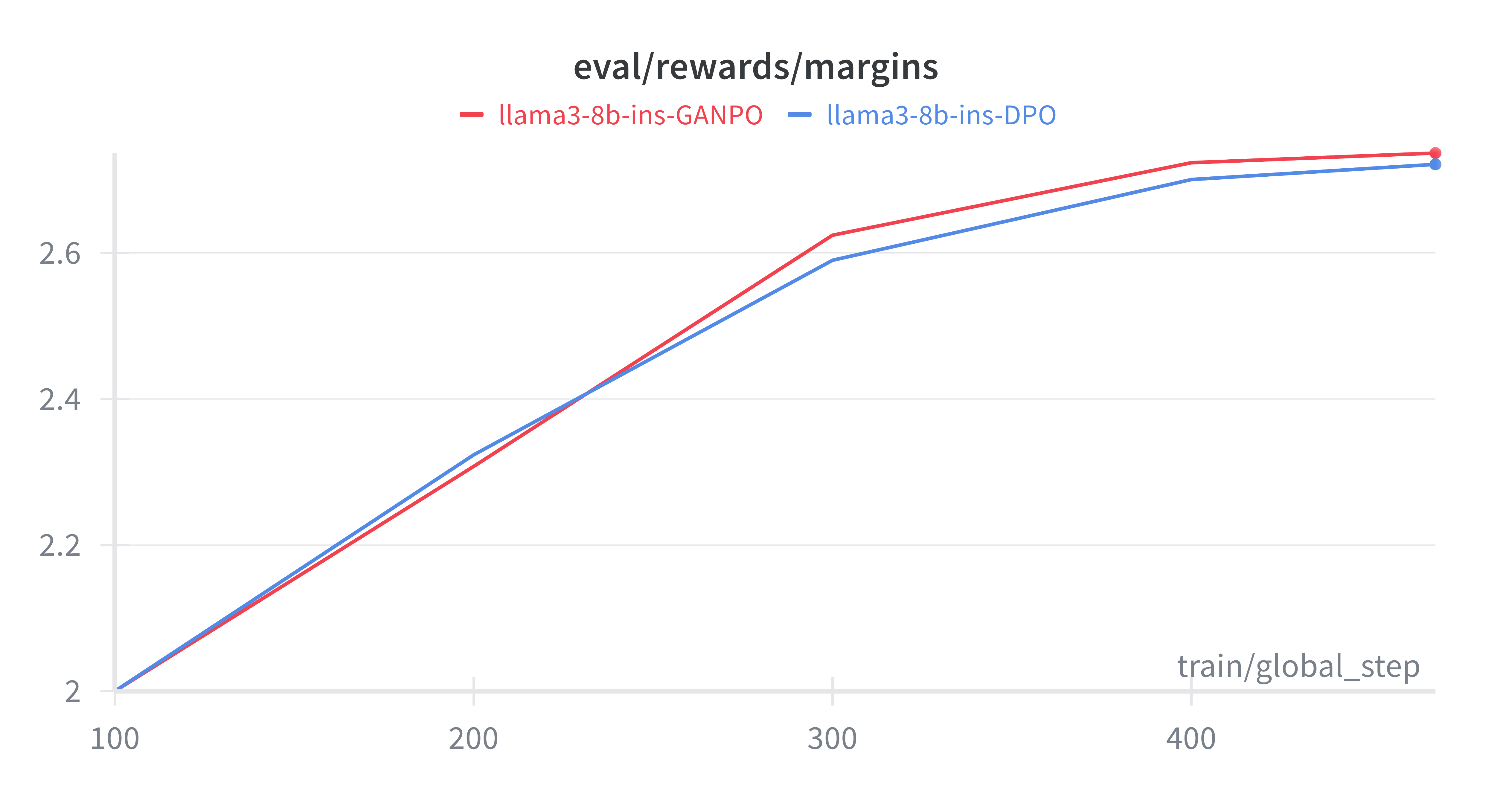}
    \caption{Evolution of reward margins during training on Llama-3-8B-Instruct, comparing DPO and GANPO. }
    \label{fig:margins_dpo_llama}
     % \vspace{-0.25in}
\end{figure}

\begin{figure}
    \centering
    \includegraphics[width=0.65\linewidth]{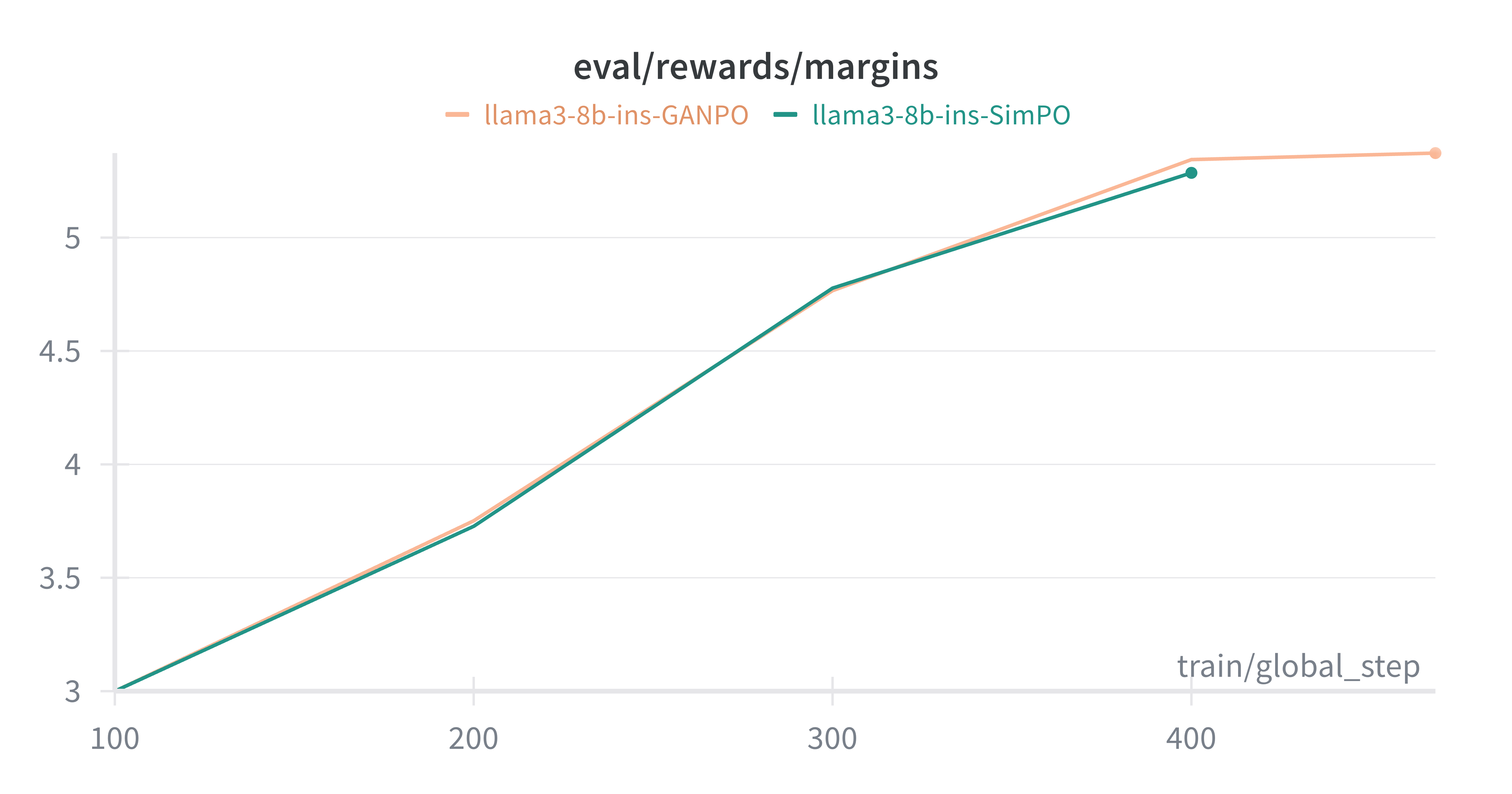}
    \caption{Reward margin comparison between SimPO and GANPO on Llama-3-8B-Instruct. }
    \label{fig:margins_simpo_llama}
     % \vspace{-0.25in}
\end{figure}

\begin{figure}[h]
\begin{lstlisting}[language=Python, basicstyle=\ttfamily\scriptsize, label={lst:discriminator}]
import torch
import torch.nn as nn
from torch.nn.utils import spectral_norm

class TransformerDiscriminator(nn.Module):
    def __init__(
        self, 
        input_dim=4096,     # Dimension of LLM hidden states
        hidden_dim=512,     # Internal dimension for Discriminator
        num_layers=2,       # Shallow Transformer depth
        num_heads=8,        # Number of attention heads
        max_seq_len=2048,   # Max supported sequence length
        dropout=0.1
    ):
        super().__init__()
        
        self.project_in = spectral_norm(nn.Linear(input_dim, hidden_dim))
        
        self.pos_embedding = nn.Parameter(
            torch.randn(1, max_seq_len, hidden_dim) * 0.02
        )
        
        encoder_layer = nn.TransformerEncoderLayer(
            d_model=hidden_dim,
            nhead=num_heads,
            dim_feedforward=hidden_dim * 4,
            dropout=dropout,
            activation='gelu',
            batch_first=True,
            norm_first=True 
        )
        self.transformer = nn.TransformerEncoder(
            encoder_layer, num_layers=num_layers
        )
        
        self.head = nn.Sequential(
            spectral_norm(nn.Linear(hidden_dim, hidden_dim)),
            nn.GELU(),
            spectral_norm(nn.Linear(hidden_dim, 1))
        )
        
        self.apply(self._init_weights)

    def _init_weights(self, m):
        if isinstance(m, nn.Linear):
            torch.nn.init.xavier_uniform_(m.weight)
            if m.bias is not None:
                nn.init.constant_(m.bias, 0)

    def forward(self, hidden_states, mask=None):
        batch_size, seq_len, _ = hidden_states.size()
        
        x = self.project_in(hidden_states) 
        
        if seq_len <= self.pos_embedding.size(1):
            x = x + self.pos_embedding[:, :seq_len, :]
        else:
            x = x + self.pos_embedding[:, :self.pos_embedding.size(1), :]
        
        src_key_padding_mask = None
        if mask is not None:
            src_key_padding_mask = (mask == 0) 
        
        x = self.transformer(x, src_key_padding_mask=src_key_padding_mask)
        
        if mask is not None:
            mask_expanded = mask.unsqueeze(-1).expand_as(x)
            sum_embeddings = torch.sum(x * mask_expanded, dim=1)
            sum_mask = mask.sum(dim=1, keepdim=True).clamp(min=1e-9)
            pooled_output = sum_embeddings / sum_mask
        else:
            pooled_output = x.mean(dim=1) 
            
        return self.head(pooled_output)
\end{lstlisting}
\caption{PyTorch implementation of the Transformer Discriminator.}\label{fig:trans-arc}
\end{figure}
%%%%%%%%%%%%%%%%%%%%%%%%%%%%%%%%%%%%%%%%%%%%%%%%%%%%%%%%%%%%%%%%%%%%%%%%%%%%%%%
%%%%%%%%%%%%%%%%%%%%%%%%%%%%%%%%%%%%%%%%%%%%%%%%%%%%%%%%%%%%%%%%%%%%%%%%%%%%%%%

\end{document}